\documentclass[twocolumn]{IEEEtran}
\usepackage[utf8]{inputenc}
\usepackage{graphicx}
\usepackage{amsmath}
\usepackage{amssymb}
\usepackage{amsthm}
\usepackage{cite}
\usepackage{caption,subcaption}
\usepackage{url}
\usepackage{xcolor}
\usepackage{mathtools,cuted,flushend}

\newtheorem{theorem}{Theorem}
\newtheorem{assumption}{Assumption}

\newtheorem{corollary}{Corollary}
\newtheorem{definition}{Definition}
\newtheorem{remark}{Remark}

\begin{document}
\title{Multi-AAV Cooperative Path Planning using Nonlinear Model Predictive Control with Localization Constraints}
	\author{Amith Manoharan, Rajnikant Sharma and P.B. Sujit %
\thanks{Amith Manoharan is a Graduate Student at IIIT Delhi, New Delhi -- 110020, India. email: amithm@iiitd.ac.in}
\thanks{Rajnikant Sharma is Assistant Professor at University of Cincinnati, Cincinnati, OH 45221. email: sharmar7@ucmail.uc.edu}
\thanks{P.B. Sujit is Associate Professor at IISER Bhopal, Bhopal -- 462066, India. email: sujit@iiserb.ac.in}}
	\maketitle
	
	\begin{abstract}
In this paper, we solve a joint  cooperative localization and path planning problem for a group of Autonomous Aerial Vehicles (AAVs) in GPS-denied areas using nonlinear model predictive control (NMPC). 
A moving horizon estimator (MHE)  is used to estimate the vehicle states with the help of relative bearing information to known landmarks and other vehicles. The goal of the NMPC is to devise optimal paths for each vehicle between a given source and destination while maintaining desired localization accuracy.  Estimating localization covariance in the NMPC is computationally intensive, hence we develop an approximate analytical closed form expression based on the relationship between covariance and path lengths to landmarks. Using this expression while computing NMPC commands reduces the computational complexity significantly.  We present numerical simulations to validate the proposed approach for different numbers of vehicles and landmark configurations. We also compare the results with EKF-based estimation to show the superiority of the proposed closed form approach.

		
	\end{abstract}
	
	Note to Practitioners:
\begin{abstract}
The use of AAVs in urban regions is expected to increase with several logistic and healthcare applications. These AAVs depend on GPS for localization, however, in urban regions, due to interference of building structures obtaining accurate localization information is difficult and at times may not be available. This issue hampers the AAV operations. In this paper, we develop a mechanism by which the AAVs use landmarks in the region and also the availability of other vehicles in the regions to localize and achieve the mission. For localization we use MHE and to generate the paths, we use a NMPC method. In order to improve computational speed, we  developed an approximate closed form analytical covariance method which is used in the NMPC for covariance calculation. We showed through several simulations that the proposed joint path planning with  localization constraints could determine optimal paths to the vehicles while satisfying the localization accuracy. This approach can be used by the UAV industries as an alternative mechanism for localization while determining the paths for the vehicles. The simulation results are promising but further work is required to experimentally demonstrate the proof-of-concept.
\end{abstract}

	
	\IEEEpeerreviewmaketitle
	
	\section{Introduction}\label{sec:intro}
\IEEEPARstart{U}{rban}  air mobility (UAM) is expected to have highly automated, cooperative, passenger, and cargo-carrying aerial vehicles in urban areas \cite{faa}, and the use of autonomous aerial vehicles (AAVs) for various activities are expected to rise substantially in the near future \cite{sesar}. Cargo delivery drones operate in urban canyons with high-rise buildings and other obstructions, which calls for significant localization accuracy. However, operating in such environments pose an additional challenge in localization since Global navigation satellite systems (GNSS) are unreliable in such scenarios. A solution to this problem is to use alternate localization schemes such as relative localization \cite{betke1997mobile,thrun1998finding}, vision-based methods \cite{lategahn2014vision,kim2005vision}, and ultra-wide-band (UWB) localization schemes \cite{krishnan2007uwb,zhang2018linear}. As the urban airspace is expected to contain a large number of AAVs, relative localization between vehicles can also be used in addition to known landmarks localization. 

Cooperative path planning with localization constraints involves the following components: (i) localization -- vehicles estimate their position by using relative measurements obtained with respect to other vehicles or landmarks, and (ii) cooperative path planning -- determine  optimal paths for each vehicle from a given source to destination. To achieve (ii), the agents must cooperate with each other to  generate motion commands that  improves the localization accuracy of the entire group while reducing the path length to reach their respected destinations. Several works have studied (i) and (ii)  separately. For instance, \cite{Kurazume1998,Roumeliotis2000,Spletzer2001,Mourikis2006,Nerurkar2009,wan2014cooperative,minaeian2016vision,frohle2018cooperative,pierre2018range,liu2018multi,guo2019ultra,zhu2019cooperative} address the problem of cooperative localization, while  \cite{zheng2005evolutionary,Agha-Mohammadi2014,Mathew2015,mac2016heuristic} focus on cooperative path planning problem. However, the collection of works that jointly address cooperative path planning with the localization constraints is limited. Below,  we will review some of the works in this domain.

\begin{figure}
	\centering 
	\begin{subfigure}{4cm}
		\includegraphics[width=\linewidth]{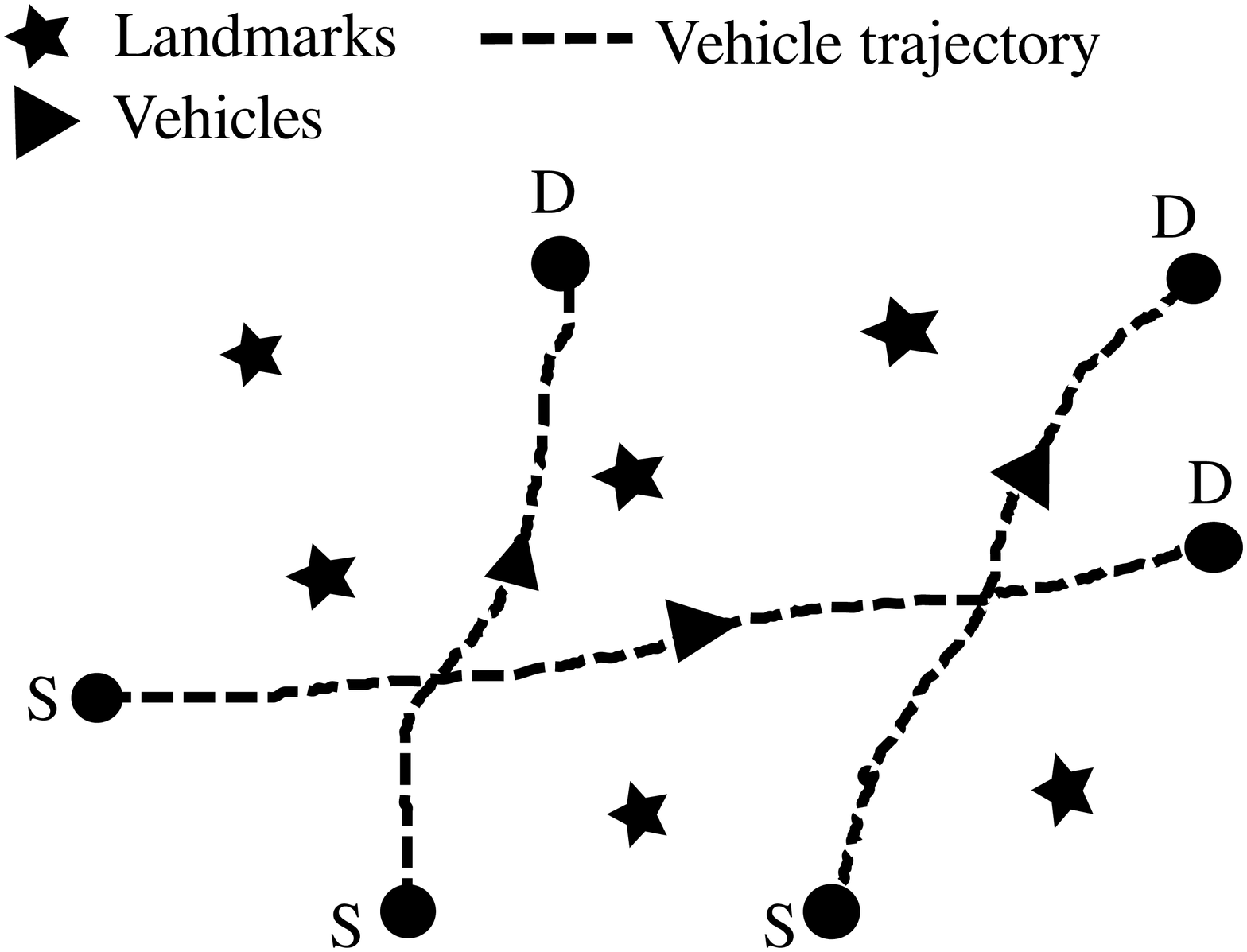}
		\caption{}
		\label{fig:scenario}
	\end{subfigure}
	\begin{subfigure}{4cm}
		\includegraphics[width=\linewidth]{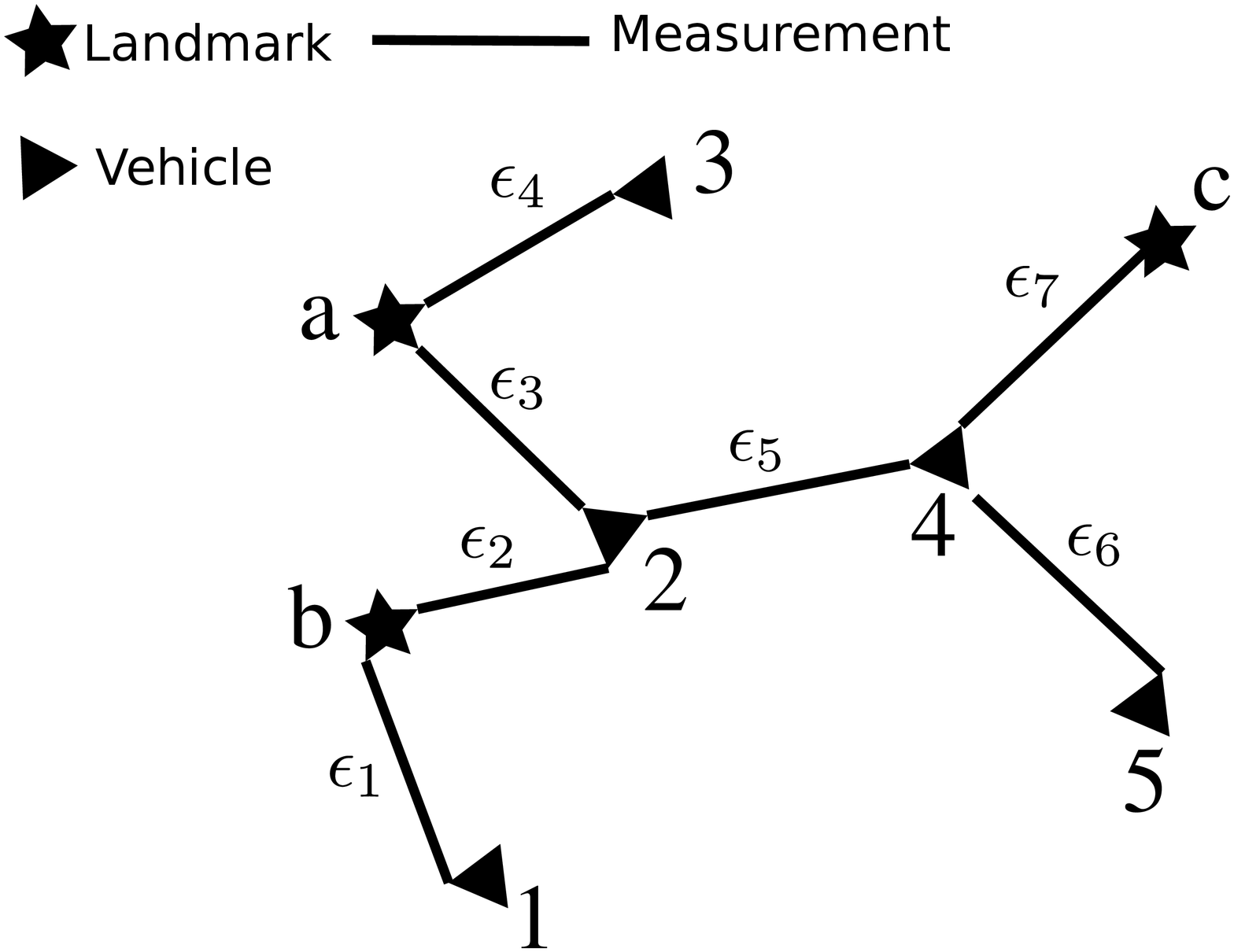}		
		\caption{}
		\label{fig:rpmg}
	\end{subfigure}
	\caption{(a) Path planning scenario. (b) Relative position measurement graph with vehicles and landmarks as nodes and measurements as edges. }
	\label{fig:path_graph}
\end{figure}  

Bopardikar et al.~\cite{bopardikar2015multiobjective} presented a graph-based probabilistic roadmap approach to tackle the path planning problem subject to localization constraints. The generated path is a discretized path while we are addressing a continuous path problem. A time-optimal path planner satisfying the covariance bounds was given in \cite{singh2016landmarks} using a swarm optimization technique coupled with a rabbit-carrot based path follower. An approach to optimally place landmarks to satisfy localization constraints was proposed in \cite{sundar2019path}. The algorithm computes an optimal path for the vehicle and the locations where the landmarks should be placed. A localizability constrained path planning method for autonomous vehicles which takes into account the laser range finder (LRF) sensor model of the vehicle is proposed in \cite{irani2018localizability} to maintain a satisfactory level of localizability throughout the path. Kassas et al. \cite{kassas2021uav} present a multi-objective motion planning algorithm in which the vehicle tries to balance the objectives of navigating to the waypoint and reducing its position estimate uncertainty. All the above works are limited to one vehicle only. 

Urban air mobility calls for improved localization accuracy due to its innate nature involving close structures, narrow pathways, and a large number of vehicles. Moving horizon estimation (MHE) has been suggested as an alternate to EKF for increasing the accuracy of nonlinear estimation problems by \cite{rawlings2006particle,wang2014optimization,mehrez2017optimization}. Erunsal et al. \cite{erunsal2019decentralized} proposed an approach combining NMPC and pose-graph-MHE for 3D formation control of micro aerial vehicles with relative sensing capability. In \cite{liu2020decentralized}, a decentralized MHE technique is proposed for networked navigation with packet dropouts.    

In this paper, we extend the  work in ~\cite{manoharan2019nonlinear} and propose a joint cooperative localization and path planning framework with MHE for estimating the vehicle position, and NMPC  framework for cooperative path planning, and a closed formulation for covariance calculation to  predict the uncertainty. This framework uses a nonlinear vehicle model in both the controller and the estimator, which mitigates the linearization errors. The analytical expression used for the covariance calculation speeds up the computations and is derived by exploiting the relationship between the vehicle-landmark path lengths to the localization uncertainty. The proposed approach provides a flexibility, where each vehicle can decide to maintain, lose, or gain connections depending on their covariance estimates.In most of the literature for multi-agent systems, the studies are formulated either as a control problem or an estimation problem \cite{Mourikis2006,Nerurkar2009,Roumeliotis2000,wan2014cooperative,wang2014optimization}. We propose a method that combines both and looks at the multi-agent problem in a holistic sense.

The major contributions of this paper as follows:
\begin{itemize}
	\item A complete framework for control and estimation of multi-vehicle cooperative path planning problem with localization constraints using NMPC and MHE.
	\item An analysis on the relation of path lengths between vehicles and landmarks on the estimation covariance
\item An approximate closed form analytical expression to compute localization error covariance
\item Evaluation of the proposed joint cooperative path planning with localization constraints framework through numerical simulations and comparison with EKF-based estimation framework
\end{itemize}

The rest of the paper is organized as follows. The problem formulation is given in Section \ref{sec:problem}. Moving horizon estimation (MHE) is explained in Section \ref{sec:MHE}. The derivation of the analytic expression for covariance is given in Section \ref{sec:analysis}. The NMPC formulation is given in Section \ref{sec:NMPC}. Simulation results are presented in Section \ref{sec:results}, and the conclusions are given in Section \ref{sec:conclusions}. 


\begin{figure*}
	\centering 
	\begin{subfigure}{7cm}
		\includegraphics[width=7cm]{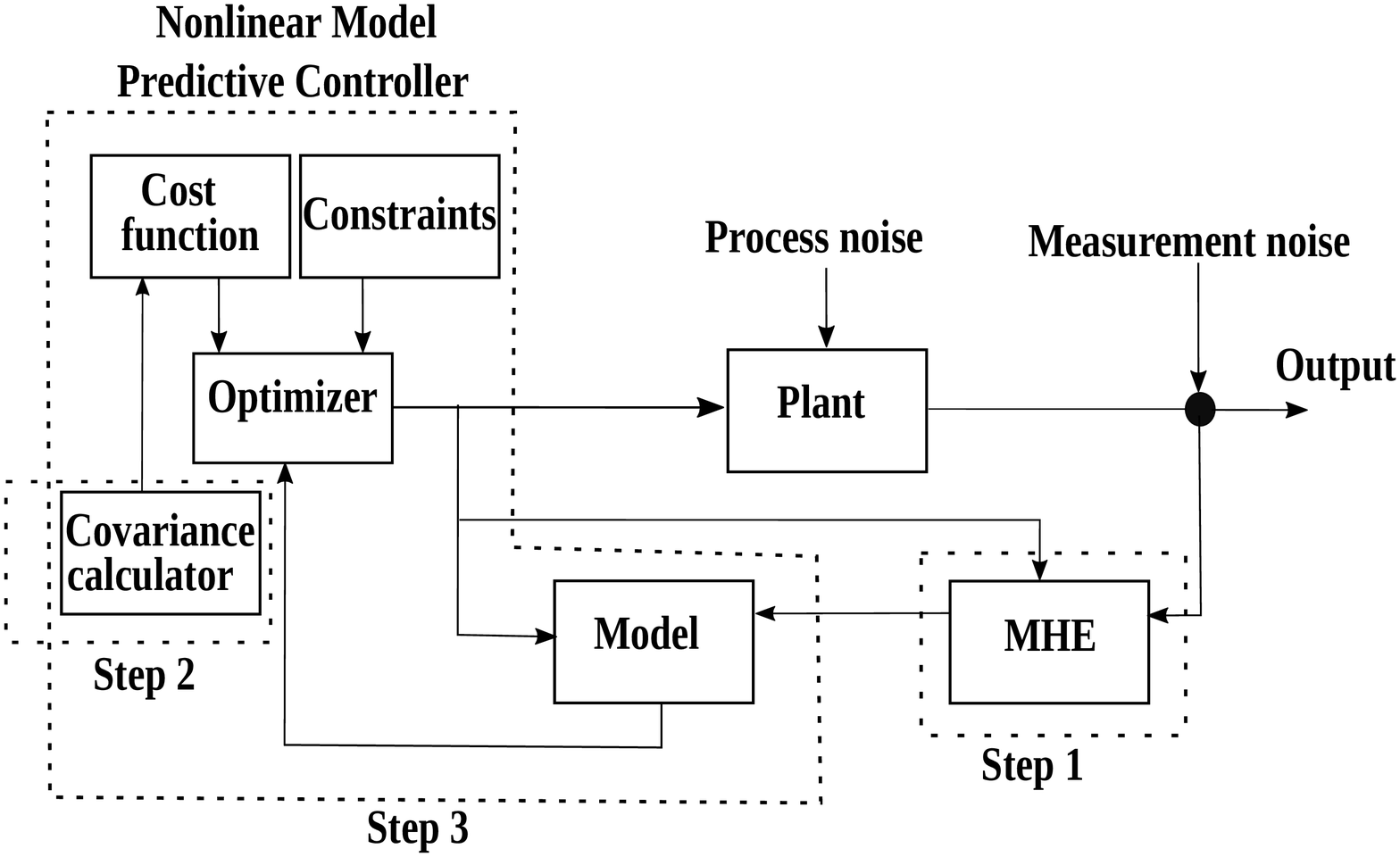}
	\end{subfigure}
	\begin{subfigure}{7cm}
		\includegraphics[width=7cm]{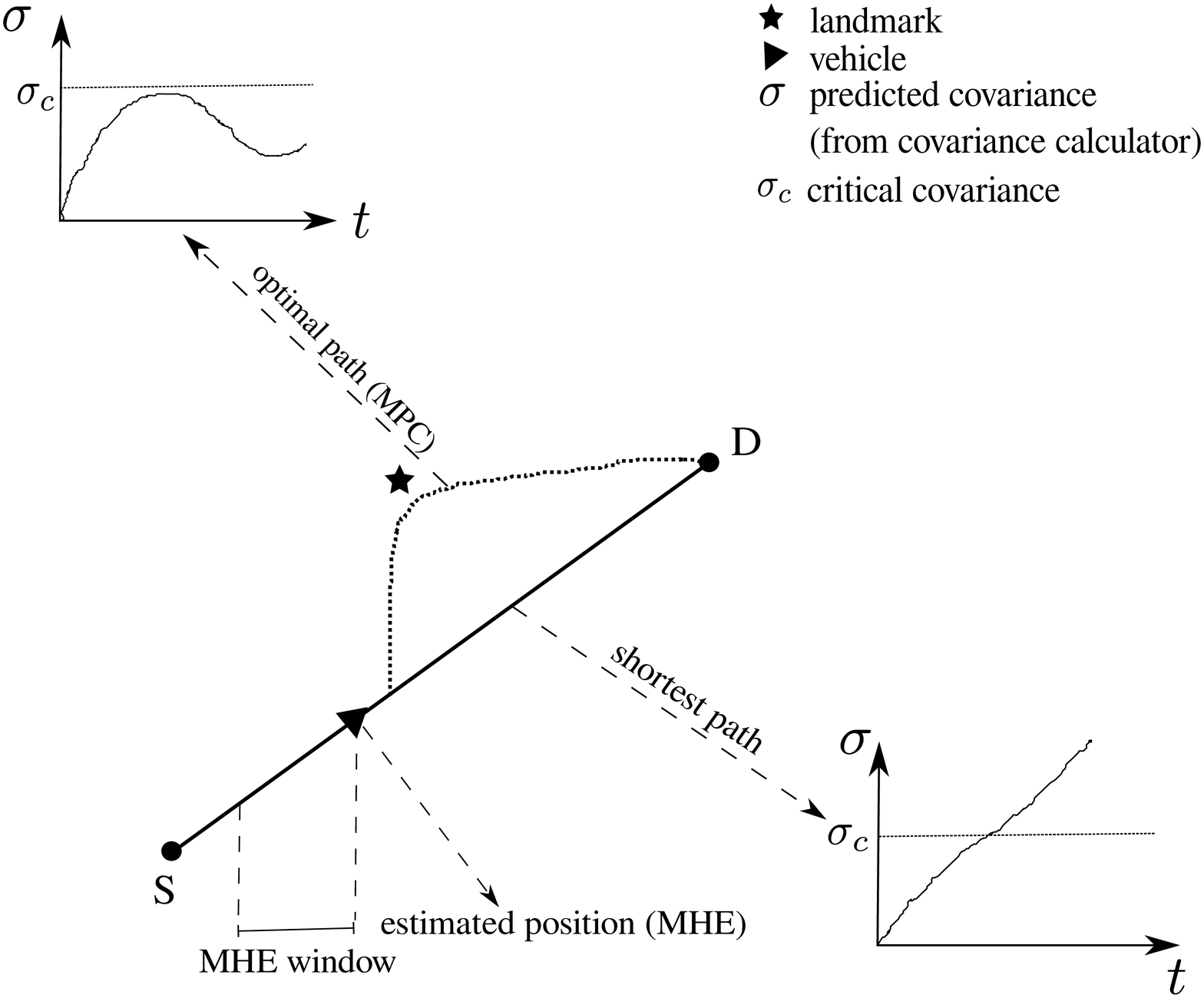}		
	\end{subfigure}\hfil 
	\caption{Block diagram and a graphical representation of the proposed NMPC-MHE control scheme. }
	\label{fig:block}
\end{figure*}  

\section{Problem formulation}\label{sec:problem}	
We consider a scenario where a group of AAVs need to navigate from their source location ($S$) to destination ($D$), as shown in Fig.~\ref{fig:scenario}. These vehicles need to transit in a GPS-denied area and we assume that any kind of GNSS are not available. In such scenarios, known landmarks or other vehicles in the area can be used for relative localization using range or bearing measurements. This structure involving landmarks and vehicles can be modeled as a dynamic relative position measurement graph (RPMG) \cite{sharma2011graph} with vehicles and landmarks as nodes and connections/measurements as edges.
\begin{definition}
	A relative position measurement graph (RPMG) for $n_v(t)$ vehicles with $n_l(t)$ landmarks is a graph $G_{n_v}^{n_l}(t)~\triangleq~\{\mathcal{V}_{n_v}^{n_l}(t),\mathcal{E}_{n_v}^{n_l}(t) \}$, where $\mathcal{V}_{n_v}^{n_l}$ is the node set consisting of $n_v(t)$ vehicle nodes and $n_l(t)$ landmark nodes (which makes a total of $n_v(t)+n_l(t)$ nodes), and $\mathcal{E}_{n_v}^{n_l}(t)$ is the edge set representing available relative measurements. The number of edges is denoted by $n_e(t)= |\mathcal{E}_{n_v}^{n_l}(t)|$.   
\end{definition}
\begin{definition}
	A path from a vehicle node to a landmark node is a finite sequence of edges which joins a sequence of distinct vertices between them. Let $\mathcal{G} \triangleq \{\mathcal{V},\mathcal{E},\phi\} $ be a graph. A path $\phi$ from vertex $i$ to vertex $j$ is a sequence of edges $\{\epsilon_1,\epsilon_2, \ldots, \epsilon_{n-1} \}$ for which there is a sequence of distinct vertices $\{\nu_1,\nu_2, \ldots, \nu_n \}$ such that $\phi \triangleq \{ \nu_1,\nu_n \}$ where $\nu_1 = i$ and $\nu_n = j$.
\end{definition}
An example RPMG ($G_{5}^{3}$ with $n_e=7$) is shown in Fig.~\ref{fig:rpmg}. A path from the vehicle $5$ to the landmark $a$ is represented by the edge set $\phi \triangleq \{\epsilon_6,\epsilon_5,\epsilon_3 \}$ which can also be represented using the vertices as $5-4-2-a$. 

Previous studies show that for cooperative localization to work, each vehicle should have a direct or indirect path to at least two known landmarks \cite{sharma2014observability}. This condition is very limiting in environments with a low number of landmarks. Hence in this paper, we find a relationship between vehicle uncertainty and path length to the landmarks and then use that relationship to formulate and solve an NMPC problem to guarantee that the covariance does not exceed a specified threshold and desired localization accuracy is achieved while performing individual missions. A moving horizon estimation (MHE) scheme is used to estimate the vehicle states. 

A graphical representation of the proposed solution using NMPC combined with MHE to tackle the cooperative localization and path planning problem is shown in Fig.~\ref{fig:block}. The optimal path is different from the shortest path since the latter may not satisfy localization constraints. The components of the block diagram are explained in the subsequent sections. In the first step of the control scheme, vehicle states are estimated by the MHE block using available measurements from the sensors (Sec.~\ref{sec:MHE}). The second step is the calculation of estimation covariances for the NMPC prediction window, which is accomplished by the covariance calculator block that contains the derived analytical expression (Sec.~\ref{sec:analysis}). In the third and final step, the NMPC controller computes the control actions for the vehicles (Sec.~\ref{sec:NMPC}).

	\section{Moving horizon estimation}\label{sec:MHE}
Moving horizon estimation (MHE) uses optimization techniques to determine state trajectories that best fit a series of measurements acquired over a finite time interval. It uses the exact nonlinear models of the available measurements and system dynamics for estimation. Also, there is another advantage of including the state/control constraints in the formulation, which helps in bounding the estimates. Like NMPC, MHE also has three main components, 1) an internal dynamic model of the process, 2) a history of past measurements, and 3) an optimization cost function over the estimation horizon.

The model used for estimation is given as:
\begin{eqnarray}
X(k)&=&f\left( X(k-1),\boldsymbol{\omega}(k),k\right) +q(k), \label{eq:f}\\
z(k)&=&h\left( X(k),\boldsymbol{\omega}(k),k\right) + \mu(k), \label{eq:z}
\end{eqnarray}	
where, $f(\cdot)$ and $h(\cdot) $ represent the state model and the observation model respectively. ${X}(k)$ and $ z(k) $ are the system states and measurements at the $k^{\mathrm{th}}$ time instant. The vectors $ q(k) $ and $ \mu(k) $ are the process and measurement noises which are assumed to be additive and zero mean white Gaussian noises with covariance $ Q $ and $ \Gamma $ respectively. $ f(\cdot) $ is defined as
	\begin{equation}
	f=\begin{bmatrix}
	x_{1}(k)\\
	y_{1}(k)\\
	\psi_{1}(k)\\
	\vdots\\
	x_{n_v}(k)\\
	y_{n_v}(k)\\
	\psi_{n_v}(k)
	\end{bmatrix}=
	\begin{bmatrix}
	x_{1}(k-1)+T_sv\cos\psi_{1}(k-1)\\
	y_{1}(k-1)+T_sv\sin\psi_{1}(k-1)\\
	\psi_{1}(k-1)+T_s\omega_1(k-1)\\
	\vdots\\
	x_{n_v}(k-1)+T_sv\cos\psi_{n_v}(k-1)\\
	y_{n_v}(k-1)+T_sv\sin\psi_{n_v}(k-1)\\
	\psi_{n_v}(k-1)+T_s\omega_{n_v}(k-1)
	\end{bmatrix},
	\end{equation}
where $ T_s $ is the sampling time used for discretization.

Let $m$ be the current time step, $N_E$ is the estimation horizon, and  we denote $\tau=m-N_E$ for simplicity. We  formulate the moving horizon estimation problem 
\begin{equation}
\min_{X}J= \lVert X_{\tau} - \hat{X}_{\tau} \rVert_{P^{-1}_{\tau}}^2  + \sum_{k=\tau}^{m} \lVert h(X_{k}) - z_{k} \rVert_{\Gamma^{-1}}^2 , 
\label{eq:MHE_cost}
\end{equation} 
	subject to:
	\begin{align*}
	X_{k+1}&= f(X_{k},\boldsymbol{\omega}_{k}),\\
	\boldsymbol{\omega}&\in\left[ \boldsymbol{\omega}^{-},\boldsymbol{\omega}^{+}\right], 
	\end{align*}
where  $\hat{X}$ is the estimated states, $P$ is the estimation covariance matrix, and $\Gamma$ is the measurement covariance. It is assumed that each vehicle can measure relative bearing to other vehicles and landmarks that are in the sensor's field-of-view ($R_s$). Relative bearing from the $i^{th}$ vehicle to the $j^{th}$ vehicle or landmark is given by the measurement model:
\begin{equation}
h(X)=\tan^{-1} \left( \frac{y_{j}-y_{i}}{x_{j}-x_{i}}\right)-\psi_i.
\end{equation} 

The first term in (\ref{eq:MHE_cost}) is known as the arrival cost and it plays an important role in stabilizing the estimator. It penalizes the deviation of the first state in the moving horizon window and its previous estimate $ \hat{X}_{\tau} $. The weighting matrix $ P $ is given by~\cite{rao2003constrained}
\begin{eqnarray}
P_{k+1}&=&Q+\nabla F_{X}( P_k-  P_k\nabla H_{X}^{T}( \nabla H_{X}P_k\nabla H_{X}^{T}+\Gamma) ^{-1}\nonumber\\
&&\nabla H_{X}P_k) \nabla F_{X}^{T},\nonumber
\end{eqnarray}
where $Q$ is the state covariance matrix, and $ \nabla F_{X},\nabla H_{X} $ are the Jacobians of $ f $ and $ h $. 
The second term in (\ref{eq:MHE_cost}) penalizes the change in predicted measurements $ h(X_{k}) $ from the actual measurements $ z_{k} $.

Now, we look into the stability of the moving horizon estimator. The following assumptions are required for proving the stability result.

\begin{assumption}
The initial state $X_0$ and the control input $\omega$ are such that, for any noise $q$, the system trajectory $X$ lies in a compact set $\chi$ and $\omega$ in a compact set $U$.
\label{assump:MHE1}	
\end{assumption}
\begin{assumption}
	The functions $f$ and $h$ are $C^2$ functions w.r.t $X$ on $ co(\chi) $ for every $\omega \in U$, where $co(\chi)$ is the convex closure of $ \chi $.
	\label{assump:MHE2}
\end{assumption}

Observation map for a horizon $N_E+1$ can be defined as 
\begin{equation}
F^{N_E}(X,\omega,q) = \begin{bmatrix}
h(X_{\tau}) \\
h \circ f^{\tau}(X_{\tau})\\
\cdot\\
\cdot\\
\cdot\\
h \circ f^{m-1} \circ \cdot \cdot \cdot f^{\tau}(X_{\tau})
\end{bmatrix},
\end{equation}
where $\circ$ is function composition. Then it is possible to re-write equation (\ref{eq:z}) as
\begin{equation}
z_{\tau} = F^{N_E}(X,\omega,q) + \mu_{\tau},
\end{equation} 
and modify the cost function as
\begin{equation}
J_m(X_{\tau,\hat{X}_{\tau}})= \lVert X_{\tau} - \hat{X}_{\tau} \rVert_{\mathbb{P}_{\tau}}^2  +\lVert F^{N_E}(X_{\tau},\omega_{\tau},q_{\tau}) - z_{\tau} \rVert_{\mathbf{\Gamma}}^2,       
\end{equation} 
where, $ \mathbf{\Gamma} = I_{N_E+1} \otimes \Gamma^{-1} $, where $\otimes$ is the Kronecker product.

Now, let's consider the following remarks:
\begin{remark}
	System (\ref{eq:f}), (\ref{eq:z}) is said to be observable in $N_E+1$ steps if there exists a $K$-function $\phi(\cdot)$ such that
	$
		\phi\left( \lVert x_1-x_2 \rVert^2 \right) \leq \lVert F^{N_E}(x_1,\omega,0) - F^{N_E}(x_2,\omega,0) \rVert^2,  
	$   
	$\forall x_1,x_2 \in \chi$ and $\forall \omega \in U^{N_E}$.
	\label{rem:MHE1}
\end{remark}
\begin{remark}
	 If the observability matrix $ \frac{\partial F^{N_E}(X,\omega,0)}{\partial X} $ has full rank, then the system is said to be observable in $N_E+1$ steps with finite sensitivity $1/\delta$ if the $K$-function $\phi(\cdot)$ satisfies the following condition
	\begin{equation}
		\delta = \inf_{x_1,x_2\in \chi;x_1 \neq x_2} \frac{\phi\left( \lVert x_1-x_2 \rVert^2 \right)}{ \lVert x_1-x_2 \rVert^2 } \geq 0.
	\end{equation}
\label{rem:MHE2}
\end{remark} 
Let $k_f$ be an upper bound on the Lipschitz constant of $f(X,\omega)$ w.r.t $X$ on $\chi$ for every $\omega \in U$ and $\mathbb{P}$ is diagonal with $ \mathbb{P} = pI_n, p>0 $. Let
\begin{equation}
r_{\mu} = \max_{\mu\in M} \lVert \mu \rVert^2,
\end{equation}
where $M$ is a compact set with $0\in M$. 

Stability of the estimator is proved using the results from \cite{alessandri1999neural,rao2003constrained,alessandri2008moving,alessandri2010advances}. 
Consider the cost function defined as: 
\begin{equation}
	J= \lVert X_{\tau} - \hat{X}_{\tau} \rVert_{\mathbb{P}_{\tau}}^2  + \sum_{k=\tau}^{m} \lVert h(X_{k}) - z_{k} \rVert_{\Gamma^{-1}}^2,
	\label{eq:stab_cost}
\end{equation}
then we can state the following theorem~\cite{alessandri2008moving,alessandri2010advances}.

\begin{theorem}
If the Assumptions \ref{assump:MHE1},\ref{assump:MHE2} are satisfied and the Remarks \ref{rem:MHE1},\ref{rem:MHE2} hold, then there exists an upper bound defined by
\begin{equation}
\lVert X_{\tau} - \hat{X}_{\tau} \rVert^2 \leq \zeta_{\tau},
\end{equation} 
where $\zeta_m$ is found using the equation
\begin{equation}
\zeta_{m+1} = \left(\frac{c_1k_fp}{p+c_2\delta}\right)\zeta_m + \left(\frac{c_3}{p+c_2\delta}\right)r_{\mu},
\label{eq:bound}
\end{equation}
$c_1,c_2$, and $c_3$ are positive constants. Let
\begin{equation}
a(p,\delta) = \frac{c_1k_fp}{(p+c_2\delta)},
\end{equation}
and if $p$ is selected such that $ a(p,\delta) <1 $, then the dynamics of (\ref{eq:bound}) is asymptotically stable. 
\end{theorem}
\begin{proof}
	A summary of the proof given by~\cite{alessandri2008moving} is detailed here for completeness. The proof is based on defining upper and lower bounds on the optimal cost $J_{m}^*$, which is the cost corresponding to an optimal estimate $\hat{X}_{\tau}^{*}$. First, the upper bound on $J_{m}^*$ should be defined. Let us define $X_{m}^o$ as the true value of the state $X$ at time $m$ and assume that $\Gamma = I$. We have that
	\begin{equation}
	    J_{m}^* \leq \lVert X_{\tau}^o - X_{\tau}^* \rVert_{\mathbb{P}_{\tau}}^2  + \sum_{k=\tau}^{m} \lVert F^{N_E}(X_{k}^o) - z_{k} \rVert^2,
	\end{equation}
	which can be modified as
	\begin{equation}
	   J_{m}^* \leq \lVert X_{\tau}^o - X_{\tau}^* \rVert_{\mathbb{P}_{\tau}}^2 + C, \label{eqn:lower_bound}  
	\end{equation}
	where $C$ is a positive constant (please see Lemma. 1 from \cite{alessandri2008moving}). Next, the upper bound on $J_{m}^*$ is defined. We can write
	{
	\begin{eqnarray}
	    \lVert F^{N_E}(X_{\tau}^o) - F^{N_E}(\hat{X}_{\tau}) \rVert^2 =&& \nonumber\\
	    \lVert [z_{\tau} - F^{N_E}(\hat{X}_{\tau})] - [z_{\tau} - F^{N_E}(X_{\tau}^o)] \rVert^2,&&\\
	    \lVert F^{N_E}(X_{\tau}^o) - F^{N_E}(\hat{X}_{\tau}) \rVert^2  \leq&& \nonumber\\ 2 \lVert z_{\tau} - F^{N_E}(\hat{X}_{\tau}) \rVert^2 + 2 \lVert z_{\tau} - F^{N_E}(X_{\tau}^o)] \rVert^2,&& \\
	    \lVert z_{\tau} - F^{N_E}(\hat{X}_{\tau}) \rVert^2  \geq&& \nonumber \\ \frac{1}{2} \lVert F^{N_E}(X_{\tau}^o) - F^{N_E}(\hat{X}_{\tau}) \rVert^2 - \lVert z_{\tau} - F^{N_E}(X_{\tau}^o)] \rVert^2.&&
	\end{eqnarray}
}
	From (\ref{eqn:lower_bound}), we can write
	\begin{equation}
	    \lVert z_{\tau} - F^{N_E}(X_{\tau}^o)] \rVert^2 < C.
	\end{equation}
	Hence, we obtain
	{	\begin{equation}
	    \lVert z_{\tau} - F^{N_E}(\hat{X}_{\tau}) \rVert^2 \geq \frac{1}{2} \lVert F^{N_E}(X_{\tau}^o) - F^{N_E}(\hat{X}_{\tau}) \rVert^2 - C. \nonumber
	\end{equation}}
	By using a similar procedure, we can write
		{
	\begin{equation}
	    \lVert X_{\tau}^* - \hat{X}_{\tau} \rVert^2 \geq \frac{1}{2} \lVert X_{\tau}^o - \hat{X}_{\tau} \rVert^2 - \lVert X_{\tau}^o - X_{\tau}^* \rVert^2. \nonumber
	\end{equation}}
	Now the upper bound can be defined as
	\begin{eqnarray}
	    J_{m}^* &\geq& \frac{1}{2} \lVert X_{\tau}^o - \hat{X}_{\tau}^* \rVert_{\mathbb{P}_{\tau}}^2 + \frac{1}{2} \lVert F^{N_E}(X_{\tau}^o) - F^{N_E}(\hat{X}_{\tau}^*) \rVert^2 - \nonumber \\
&&	    \lVert X_{\tau}^o - X_{\tau}^* \rVert_{\mathbb{P}_{\tau}}^2 - C. \label{eqn:upper_bound}
	\end{eqnarray}
	Now, by combining the bounds (\ref{eqn:lower_bound}) and (\ref{eqn:upper_bound}), and rewriting we get
{
	\begin{eqnarray}
	    \frac{1}{2} \lVert X_{\tau}^o - \hat{X}_{\tau}^* \rVert_{\mathbb{P}_{\tau}}^2 +  \frac{1}{2} \lVert F^{N_E}(X_{\tau}^o) -  F^{N_E}(\hat{X}_{\tau}^*) \rVert^2 \nonumber && \\\leq 2 \lVert X_{\tau}^o - X_{\tau}^* \rVert_{\mathbb{P}_{\tau}}^2 + 2C.\nonumber
	\end{eqnarray}}
	According to the Remarks \ref{rem:MHE1} and \ref{rem:MHE2}, the above equation can be written as
	{
	\begin{equation*}
	    \lVert F^{N_E}(X_{\tau}^o) - F^{N_E}(\hat{X}_{\tau}^*) \rVert^2 = \phi \left( \lVert X_{\tau}^o - \hat{X}_{\tau}^* \rVert^2 \right),
	\end{equation*}}
	and
	\begin{equation*}
	    \delta \lVert X_{\tau}^o - \hat{X}_{\tau}^* \rVert^2 \leq \phi \left( \lVert X_{\tau}^o - \hat{X}_{\tau}^* \rVert^2 \right).\nonumber
	\end{equation*}
	Now, it is possible to define the bound on the estimation error as
	\begin{equation}
	    \lVert X_{\tau}^o - \hat{X}_{\tau}^* \rVert^2 \leq \frac{4p}{p+\delta}\lVert X_{\tau}^o - X_{\tau}^* \rVert^2 + \frac{4}{p+\delta}C.\nonumber
	\end{equation}
	Using the Lipschitz continuity of $f(\cdot)$, it can be written that
	\begin{equation}
	    \lVert X_{\tau}^o - X_{\tau}^* \rVert^2 = 2k_f \lVert X_{\tau-1}^o - \hat{X}_{\tau-1}^* \rVert^2 + 2r_{\mu}.\nonumber
	\end{equation}
Hence,
\begin{equation*}
    \lVert X_{\tau}^o - \hat{X}_{\tau}^* \rVert^2 \leq \zeta_{\tau}.
\end{equation*}
It can also be deduced that if $\zeta_m < \zeta_{m-1}$ and $a(p,\delta)<1$, then $\zeta_m$ tends to $\frac{\beta}{1-a(p,\delta)}$ as $m \longrightarrow +\infty$ (please see Theorem~1 from~\cite{alessandri1999neural}), where $\beta~=~ \left(\frac{c_3}{p+c_2\delta}\right)r_{\mu} $. 	
\end{proof}

The following section presents the derivation of the analytical expression for calculating the covariances using the path information. This result will be later used for predicting covariances for the NMPC cost function.
	\section{Covariance calculation}\label{sec:analysis}


Consider an example configuration of two vehicles as shown   in Fig.~\ref{fig:3a}(i), where the vehicles are represented by $1$ and $2$ and two landmarks by $a$ and $b$. In order to understand how the paths/connections/measurements from a landmark to a vehicle influence the uncertainty of the vehicle states, we consider each landmark separately and analyze. Consider the segment $a-1-2$ of the graph in Fig.~\ref{fig:3a}(i), as shown in Fig.~\ref{fig:3a}(ii). The observability matrix for the system can be written as
\begin{equation}
O = \begin{bmatrix}
oa1 & 0 \\
o12 & -o12
\end{bmatrix},
\end{equation}
where $oa1,o12$, and $-o12$ are the derivatives of the measurements with respect to the vehicle states. For example, $oa1$ is the derivative of the measurement between the landmark $a$ and the vehicle-$1$ with respect to the vehicle-$1$. Since there is no measurement between the landmark $a$ and the vehicle-$2$, the corresponding entry ($O_{12}$) is zero. Note that the size of $O$ depends on the number of edges and it may not be a square matrix. Assuming the measurement covariance matrix $ \Gamma = I $, and zero-mean white Gaussian noise, the observability grammian is defined as $ O^TO $ and the covariance matrix $ P $ is written as 
\begin{eqnarray}
P &\leq& (O^T\Gamma^{-1}O)^{-1}, \\
 &\leq& \begin{bmatrix}
\frac{1}{oa1^2} & \frac{1}{oa1^2} \\
\frac{1}{oa1^2} & \frac{1}{o12^2}+\frac{1}{oa1^2} \label{eq:p2a}
\end{bmatrix}.
\end{eqnarray}
The first element of the $P$ matrix corresponds to the vehicle-$1$ connecting to the landmark $a$, hence $oa1$ (let us discard the square and fraction for easy understanding). The element $ \frac{1}{o12^2}+\frac{1}{oa1^2} $ of the $P$ matrix corresponds to the vehicle-$2$. Observing that it is connected to the landmark $a$ through vehicle-$1$, we can see both $oa1$ and $o12$ are present in the entry. Next, we consider the section $1-2-b$ of the graph, as shown in Fig.~\ref{fig:3a}(iii). The observability and covariance matrices for the system are written as
\begin{eqnarray}
O = \begin{bmatrix}
0 & ob2 \\
o12 & -o12 
\end{bmatrix},
P \leq \begin{bmatrix}
\frac{1}{o12^2}+\frac{1}{ob2^2} & \frac{1}{ob2^2} \\
\frac{1}{ob2^2} & \frac{1}{ob2^2}  \label{eq:p2b}
\end{bmatrix}.
\end{eqnarray}
The first element of $P$ indicates that vehicle-$1$ is connected to the landmark $b$ through vehicle-$2$. The last entry shows that vehicle-$2$ is directly connected to the landmark $b$ and hence  only $ob2$ is present. 
	\begin{figure*}
		\centering
		\begin{subfigure}{4cm}
		\includegraphics[width=4cm]{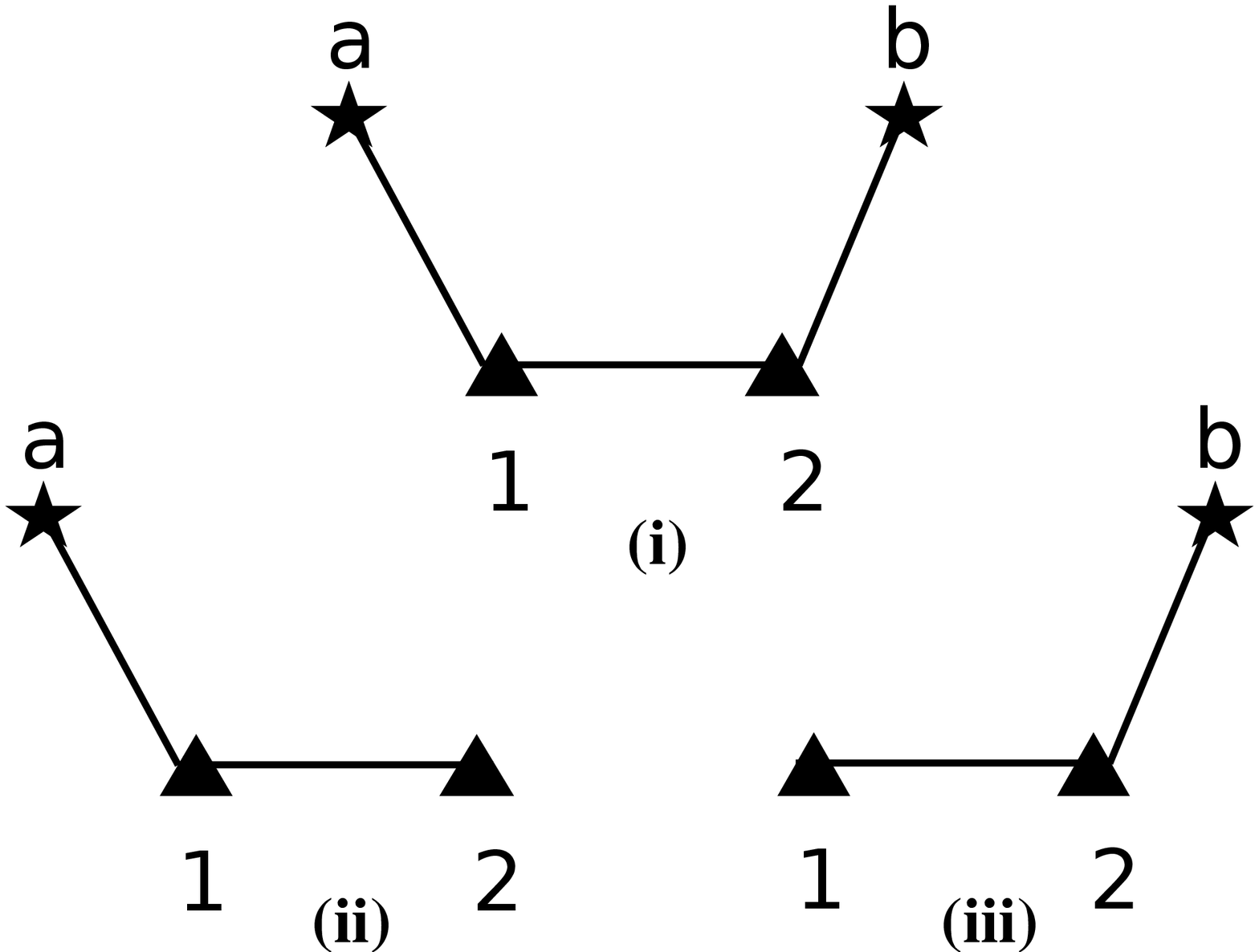}
		\caption{}\label{fig:3a}
		\end{subfigure}\hspace{10mm}
			\begin{subfigure}{4cm}
		\includegraphics[width=4cm]{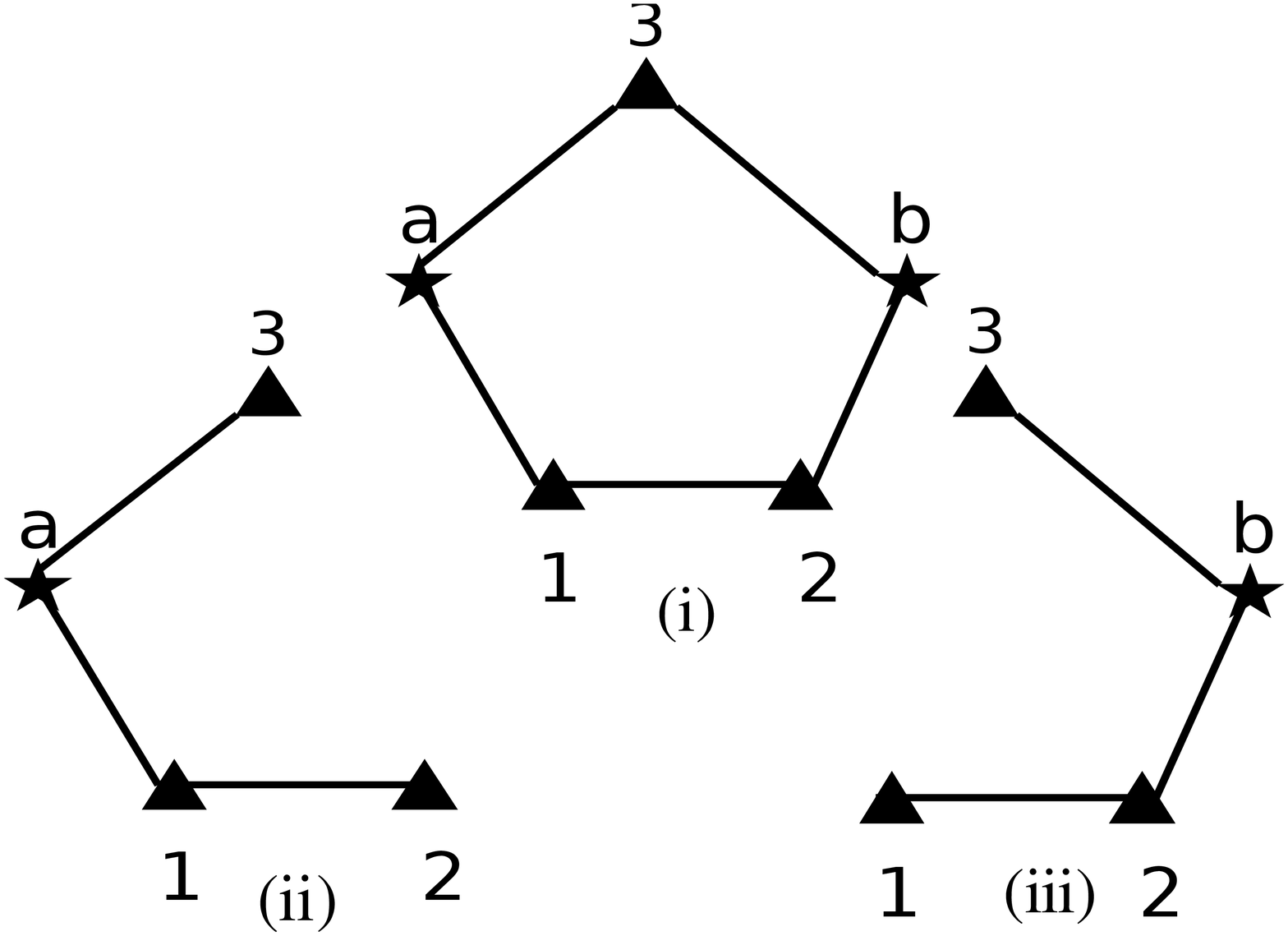}
				\caption{} \label{fig:3b}
		\end{subfigure}\hspace{10mm}
	\begin{subfigure}{4cm}
		\includegraphics[width=6cm]{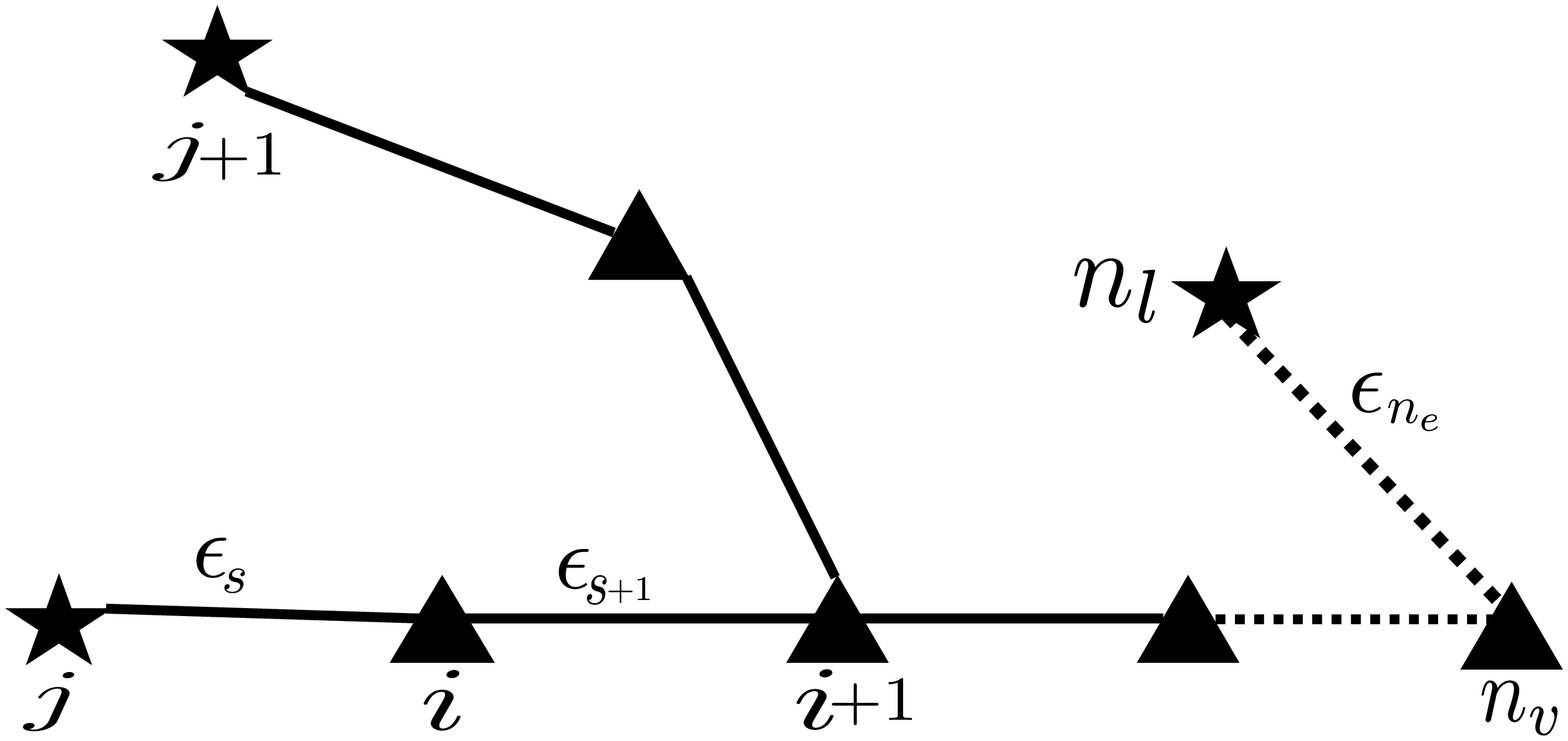}
				\caption{}
		\label{fig:config5}
	\end{subfigure}
		\caption{(a) Different configurations of two vehicles and two landmarks. (b) Different configuration of the system with 3 vehicles and 2 landmarks 
			(c) Notations for a general multi-vehicle-landmark RPMG.}
		\label{fig:config2}
	\end{figure*}


Next, we extend similar analysis for a three vehicle configuration  as shown in Fig.~\ref{fig:3b}(i). Consider the section $3-a-1-2$ of the graph, shown in Fig.~\ref{fig:3b}(ii). The observability matrix and covariance matrix $P $ for this configuration are 
\begin{eqnarray}
O = \begin{bmatrix}
oa1 & 0 & 0 \\
o12 & -o12 & 0\\
0 & 0 & oa3
\end{bmatrix},
\label{eq:ob_config3}
P \leq \begin{bmatrix}
\frac{1}{oa1^2} & \frac{1}{oa1^2}  & 0\\
\frac{1}{oa1^2} & \frac{1}{o12^2}+\frac{1}{oa1^2} & 0\\
0 & 0 & \frac{1}{oa3^2}  \label{eq:p3a}
\end{bmatrix}.\nonumber
\end{eqnarray} 
The first element of the $P$ matrix contain only $oa1$ since the vehicle-$1$ is directly connected to the landmark $a$. The vehicle-$2$ is connected to the landmark $a$ through vehicle-$1$, hence both $oa1$ and $o12$ can be seen in the corresponding entry of the $P$ matrix. Since vehicle-$3$ is directly connected to the landmark, the last element of $P$ contains only $oa3$ as expected. Now, let us take the section $1-2-b-3$ of the graph, shown in Fig.~\ref{fig:3b}(iii). The observability and covariance matrices are given by
\begin{eqnarray}
O = \begin{bmatrix}
0 & ob2 & 0 \\
o12 & -o12 & 0 \\
0 & 0 & ob3
\end{bmatrix},
P \leq \begin{bmatrix}
\frac{1}{o12^2}+\frac{1}{ob2^2} & \frac{1}{ob2^2} & 0 \\
\frac{1}{ob2^2} & \frac{1}{ob2^2} & 0\\
0 & 0 & \frac{1}{ob3^2}  \label{eq:p3b}
\end{bmatrix}.\nonumber
\end{eqnarray}
Vehicle-$2$ and vehicle-$3$ are directly connected to the landmark $b$. Therefore, the corresponding entries in the $P$ matrix contains only $ob2$ and $ob3$. The vehicle-$1$ is connected to the landmark $b$ through vehicle-$2$ and this information is clearly reflected in the first entry of $P$.

Using the above example configuration analysis, we can further extend the analysis to a general result with $n_v$ vehicle and $n_l$ landmarks as shown in Fig.~\ref{fig:config5}. The landmarks are represented using stars, and the vehicles are represented using triangles. All the vehicles are denoted by  $i,i+1,\ldots,n_v$, landmarks as $j,j+1,\ldots,n_l$, and edges connecting the vehicles and landmarks as $\epsilon_s,\epsilon_{s+1},\ldots,\epsilon_{n_e}$, where $i=j=s=1$. The observability vector associated with an edge/measurement is represented using $\epsilon$ with the edge number as subscript for simplicity. For example, the observability vector between landmark-$j$ and vehicle-$i$, $oji$ is represented by $\epsilon_s$. Now, the following theorem for a general RPMG can be stated. 

\begin{theorem}\label{thm:covariance}
The covariance associated with the vehicle-$i$, $i=1,\ldots,n_v$, due to the landmark-$j$, $j=1,\ldots,n_l$, is given by 
\begin{equation}
    p_{ij} = \sum_{s \in \mathcal{S}} \frac{1}{\epsilon_s},
\end{equation}
where $s=1,\ldots,n_e$, $\mathcal{S}$ is the set of edges that forms a path from vehicle-$i$ to the landmark-$j$, $n_v$ is the number of vehicles, $n_l$ is the number of landmarks, and $n_e$ is the number of edges in the RPMG.
\end{theorem}
\begin{proof}
The covariances associated with each vehicle for a two-vehicle-two-landmark configuration is given by equations \eqref{eq:p2a} and \eqref{eq:p2b}, followed by three vehicles in \eqref{eq:p3a} and \eqref{eq:p3b}. The relation given in Theorem~\ref{thm:covariance} is clearly reflected in the elements of the corresponding covariance matrices. The generalization to $n_v$ vehicles and $n_l$ landmarks is straightforward from the previous analysis.

However, we prove the theorem through  contradiction. Consider Fig.~\ref{fig:3b}(ii). According to Theorem~\ref{thm:covariance}, the term/edge $\frac{1}{oa1^2}$ should be present in $p_{2a}$. Suppose we write the $P$ matrix without that term. The new $P$ matrix and the original $P$ matrix found by observability analysis given by \eqref{eq:p3a} is written side-by-side showing only the element corresponding to $p_{2a}$.  
\begin{equation}
 \begin{bmatrix}
	\cdots & \cdots  & \cdots\\
	\cdots & \frac{1}{o12^2} & \cdots\\
	\cdots & \cdots & \cdots  
\end{bmatrix},
 \begin{bmatrix}
	\cdots & \cdots  & \cdots\\
	\cdots & \frac{1}{o12^2}+\frac{1}{oa1^2} & \cdots\\
	\cdots & \cdots & \cdots  \label{eq:old_new_p3a}
\end{bmatrix}.
\end{equation}

The new $O^TO$ matrix for the system found by inverting the first $P$ matrix will be
\begin{equation}
	O^TO = \begin{bmatrix}
		\frac{oa1^4}{-o12^2+oa1^2} & \frac{o12^2 oa1^2}{o12^2-oa1^2}  & 0\\
		\frac{o12^2 oa1^2}{o12^2-oa1^2} & \frac{o12^2 oa1^2}{-o12^2+oa1^2} & 0\\
		0 & 0 & \frac{1}{oa3^2}  \label{eq:proof1}
	\end{bmatrix},
\end{equation}
which contradicts with the $O^TO$ matrix derived from equation~ \eqref{eq:ob_config3}, which is 
\begin{equation}
 \begin{bmatrix}
		o12^2+oa1^2 & -o12^2  & 0\\
		-o12^2& o12^2 & 0\\
		0 & 0 & oa3^2  
	\end{bmatrix}.
\label{eq:oto_proof}
\end{equation}
Similarly, if an additional term $\frac{1}{oa3^2}$ is present in $p_{2a}$, the corresponding $O^TO$ matrix for the system will be
\begin{equation}
	O^TO = \begin{bmatrix}
		oa1^2+\frac{o12^2oa3^2}{o12^2+oa3^2} & \frac{o12^2 oa3^2}{o12^2+oa3^2}  & 0\\
		-\frac{o12^2 oa3^2}{o12^2+oa3^2} & \frac{o12^2 oa3^2}{o12^2+oa3^2} & 0\\
		0 & 0 & oa3^2  \label{eq:proof2}
	\end{bmatrix},
\end{equation}
which also contradicts with the $O^TO$ matrix given in~\eqref{eq:oto_proof}.
Hence, the relation given in Theorem~\ref{thm:covariance} is always true.
\end{proof}

The following corollaries can be written from Theorem~\ref{thm:covariance}.

\begin{corollary}\label{cor:path}
If there is more than one path from a landmark to a vehicle, and these paths are numbered from $1$ to $g_{ij}$, where $g_{ij}$ is the total number of paths from the landmark-$j$ to the vehicle-$i$, then the total covariance of the vehicle is given by
\begin{equation}
	 p_{ij} = \sum_{\kappa=1}^{g_{ij}} p_{ij}^{\kappa},
\end{equation}
where $i=1,\ldots,n_v$, $j=1,\ldots,n_l$, and $p_{ij}^{\kappa}$ is the covariance of the vehicle-$i$ due to the landmark-$j$ considering only the path-$\kappa$.
\end{corollary}

\begin{corollary}\label{cor:landmark}
If there is more than one landmark connected to a vehicle, then the covariance of the vehicle is given by
\begin{equation}
	 p_i  = \sum_{j \in \mathcal{J}} p_{ij},
\end{equation}
where $i=1,\ldots,n_v$, $j=1,\ldots,n_l$, and $\mathcal{J}$ is the set of landmarks connected to the vehicle-$i$.
\end{corollary}

\begin{proof}
The method to calculate covariances for Corollary~\ref{cor:path} and \ref{cor:landmark} is given in Theorem~\ref{thm:covariance}. The summation is based on the properties of the information matrix given as follows~\cite{fisher,zegers2015fisher} 

If $X=(X_1,X_2,\ldots,X_n)$ and $X_1,X_2,\ldots,X_n$ are independent random variables, then $I_X(\alpha)=I_{X_1}(\alpha)+I_{X_2}(\alpha)+\ldots I_{X_n}(\alpha)$, where $I_x(\alpha)$ is the information matrix defined as 
\begin{equation*}
	I_x(\alpha) = E_{\alpha}\left[ \left( \frac{\partial}{\partial\alpha}\log f(X|\alpha)\right) ^2\right] = \mathrm{Var}_{\alpha}\left(  \frac{\partial}{\partial\alpha}\log f(X|\alpha)\right).
\end{equation*}
Since \[f(x|\alpha) = \prod_{i=1}^{n}f_i(x_i|\alpha),\] where $f_i(\cdot|\alpha)$ is the pdf of $X_i$,
\begin{eqnarray}
\mathrm{Var}\left[  \frac{\partial}{\partial\alpha}\log f(X|\alpha)\right] &=& \sum_{i=1}^{n}\mathrm{Var}\left[  \frac{\partial}{\partial\alpha}\log f_i(X_i|\alpha)\right], \nonumber \\
I_X(\alpha) &=& \sum_{i=1}^{n}I_{X_i}(\alpha). \nonumber 
\end{eqnarray} 

 Since covariance is the inverse of information, we can find the total covariance associated with each vehicle by adding the components from all the paths and landmarks.
\end{proof}

In the next sub-sections, we show how this information can be used to analyze the evolution of covariance in multi-vehicle-landmark systems with range and bearing measurements. 

\subsection{Range measurements}
Consider the configuration given in Fig.~\ref{fig:config2}(i). The vehicle kinematics are defined as
\begin{eqnarray}
\dot{x}_i &=& v \cos{\psi_1}, \nonumber \\
\dot{y}_i &=& v \sin{\psi_1}, 
\end{eqnarray} 
where $i=1,2$ with range measurements
\begin{eqnarray}
h_{1a} &=& \sqrt{(x_1 - x_a)^2 + (y_1 - y_a)^2}, \\
h_{12} &=& \sqrt{(x_1 - x_2)^2 + (y_1 - y_2)^2}, \\
h_{2b} &=& \sqrt{(x_2 - x_b)^2 + (y_2 - y_b)^2},
\end{eqnarray}
and we derive the observability matrix using the Lie derivatives~ \cite{sharma2011graph}. To simplify the representation we denote $x_{1a}=(x_1 - x_a),x_{12}=(x_1 - x_2) ,x_{2b}=(x_2 - x_b)$, $y_{1a}=(y_1 -y_a),y_{12}=(y_1 - y_2) ,y_{2b}=(y_2 - y_b)$. 
Define 
\begin{equation}
f_L = \begin{bmatrix}
\cos{\psi_1} 
\sin{\psi_1} 
\cos{\psi_2} 
\sin{\psi_2} 
\end{bmatrix}',
\end{equation}
and the vehicle kinematics can be represented as
\begin{equation}
\dot{X} = vf_L.
\end{equation}
The gradient of zero\textsuperscript{th} order Lie derivatives are given as
\begin{eqnarray}
H_{1a} = \begin{bmatrix}
\frac{x_{1a}}{R_{1a}} & \frac{y_{1a}}{R_{1a}} & 0 & 0 
\end{bmatrix}, 
H_{2b} = \begin{bmatrix}
& 0 & 0 & \frac{x_{2b}}{R_{2b}} & \frac{y_{2b}}{R_{2b}}  
\end{bmatrix},\nonumber\\
H_{12} =\begin{bmatrix}
\frac{x_{12}}{R_{12}} & \frac{y_{12}}{R_{12}} & \frac{-(x_{12})}{R_{12}} & \frac{-(y_{12})}{R_{12}} 
\end{bmatrix}, \nonumber
\end{eqnarray}
where $R_{(\cdot)} $ is the distance between nodes.
The gradient of first order Lie derivatives are given as

		\begin{align}
		\frac{\partial }{\partial X}\left( \frac{\partial h_{1a}}{\partial X}  \cdot f \right) &=& \begin{bmatrix}
		\frac{(y_{1a})^2 C\psi_1 -(x_{1a})(y_{1a}) S\psi_1}{R_{1a}^3} \\
		\frac{(x_{1a})^2 S\psi_1-(x_{1a})(y_{1a}) C\psi_1}{R_{1a}^3}\\ 0 \\ 0 
		\end{bmatrix},\nonumber
				\end{align}
				\begin{align}
		\frac{\partial }{\partial X}\left( \frac{\partial h_{2b}}{\partial X}  \cdot f \right) &=& \begin{bmatrix}
		0 \\ 0 \\ \frac{(y_{2b})^2 C\psi_2 + (-x_{2b})(y_{2b}) S\psi_2}{R_{2b}^3} \\
		\frac{(x_{2b})^2 S\psi_2 + (-y_{2b})(x_{2b}) C\psi_2}{R_{2b}^3} 
		\end{bmatrix},\nonumber
				\end{align}{
						\begin{align}
		\frac{\partial }{\partial X}\left( \frac{\partial h_{12}}{\partial X}  \cdot f \right) &=& 
	\begin{bmatrix}
		\frac{-2(y_{12}) S(\frac{\Delta\psi^-}{2})((x_{12}) C(\frac{\Delta\psi^+}{2}) + (y_{12}) S(\frac{\Delta\psi^+}{2}))}{R_{12}^3} \\
		\frac{2(x_{12}) S(\frac{\Delta\psi^-}{2})((x_{12}) C(\frac{\Delta\psi^+}{2}) + (y_{12}) S(\frac{\Delta\psi^+}{2}))}{R_{12}^3} \\
		\frac{2(y_{12}) S(\frac{\Delta\psi^-}{2})((x_{12}) C(\frac{\Delta\psi^+}{2}) + (y_{12}) S(\frac{\Delta\psi^+}{2}))}{R_{12}^3} \\
		\frac{-2(x_{12}) S(\frac{\Delta\psi^-}{2})((x_{12}) C(\frac{\Delta\psi^+}{2}) + (y_{12}) S(\frac{\Delta\psi^+}{2}))}{R_{12}^3} \end{bmatrix}\nonumber,
		\end{align}}
where $ \sin $ and $ \cos $ are abbreviated as $S$ and $C$, $ \Delta\psi^- = \psi_1-\psi_2 $, and $ \Delta\psi^+ = \psi_1+\psi_2 $.

Observability matrix is formed by using Lie derivatives up to first order as
\begin{equation}
O = \begin{bmatrix}
\nabla L^0 \\
\nabla L^1
\end{bmatrix},
\end{equation} 
where $L^0$ and $L^1$ are the zero\textsuperscript{th} and first order Lie derivatives respectively. The covariance matrix $P$ is found by inverting $ O^TO $ with the assumption of $\Gamma=I$.

The standard deviation in $x$ direction for the first vehicle can be found by taking the square root of the first element of the $P$ matrix,
\begin{equation}
\sigma_{x_1} = \sqrt{P(1,1)}.
\end{equation} 
Similarly, the standard deviation in the $y$ direction and the combined position uncertainty can be found as
\begin{eqnarray}
\sigma_{y_1} &=& \sqrt{P(2,2)},\\
\sigma_{p_1} &=& \sqrt{\sigma_{x_1}^2 + \sigma_{y_1}^2}.
\end{eqnarray}

Since the derived $P$ matrix is very large with several terms, we consider some simplifying assumptions to formulate an approximate relation. All the distance terms were substituted with a single average value. The resulting relation is given as
\begin{equation}
\sigma_{p_1} = \sqrt{\frac{2}{3}+R_g^2 \csc^2{(\psi_1-\theta_g)}},
\label{eq:s_range}
\end{equation}
where $R_g,\theta_g$ are the average value of distance and LOS angle. It is evident from the relation that the covariance of vehicle depends on the distances from the landmarks and the LOS angles to them. As the distance increases, the covariance also increases. Similarly, $ \sigma_p $ of other vehicles can also be found. 
\subsection{Bearing measurements}\label{sec:b_measurement}
Let's extend the vehicle model to contain three states. The new model is given as
\begin{eqnarray}
\dot{x}_i &=& v \cos{\psi_i}, \nonumber\\
\dot{y}_1 &=& v \sin{\psi_i}, \nonumber\\
\dot{\psi}_i &=& \omega_i, \nonumber 
\end{eqnarray}
where vehicle $i=1,2$. With the same configuration as in Fig.~\ref{fig:config2}(i), define bearing measurement equations as
\begin{eqnarray}
h_{a1} &=& \arctan\left( \frac{y_a-y_1}{x_a-x_1}\right) - \psi_1, \\
h_{b2} &=& \arctan\left( \frac{y_b-y_2}{x_b-x_2}\right) - \psi_2, \\
h_{12} &=& \arctan\left( \frac{y_1-y_2}{x_1-x_2}\right) - \psi_2, \\
h_{21} &=& \arctan\left( \frac{y_2-y_1}{x_2-x_1}\right) - \psi_1, 
\end{eqnarray}
the gradient of zeroth order Lie derivatives are given as
\begin{eqnarray}
H_{a1} &=& \begin{bmatrix}
\frac{-(y_1 - y_a)}{R_{a1}^2} & \frac{x_1 - x_a}{R_{a1}^2} & -1 & 0 & 0 & 0 
\end{bmatrix}, \\
H_{b2} &=& \begin{bmatrix}
0&0&0&\frac{-(y_2 - y_b)}{R_{b2}^2} & \frac{x_2 - x_b}{R_{b2}^2} & -1  
\end{bmatrix}, \\
H_{12} &=& \begin{bmatrix}
\frac{y_2 - y_1}{R_{12}^2} & \frac{x_1 - x_2}{R_{12}^2} & 0 & \frac{y_1 - y_2}{R_{12}^2} & \frac{x_2 - x_1}{R_{12}^2} & -1 
\end{bmatrix}, \hspace{6mm} \\
H_{21} &=& \begin{bmatrix}
\frac{y_2 - y_1}{R_{21}^2} & \frac{x_1 - x_2}{R_{21}^2} & -1 & \frac{y_1 - y_2}{R_{21}^2} & \frac{x_2 - x_1}{R_{21}^2} & 0 
\end{bmatrix}, \hspace{6mm}
\end{eqnarray} 
using the geometry, the equations are changed to make it in terms of LOS angles as follows
\begin{eqnarray}
H_{a1} &=& \begin{bmatrix}
\frac{\sin\theta_{a1}}{R_{a1}} & \frac{-\cos\theta_{a1}}{R_{a1}} & -1 & 0 & 0 & 0 
\end{bmatrix}, \\
H_{b2} &=& \begin{bmatrix}
0&0&0&\frac{\sin\theta_{b2}}{R_{b2}} & \frac{-\cos\theta_{b2}}{R_{b2}} & -1 
\end{bmatrix}, \\
H_{12} &=& \begin{bmatrix}
\frac{-\sin\theta_{12}}{R_{12}} & \frac{\cos\theta_{12}}{R_{12}} & 0 & \frac{\sin\theta_{12}}{R_{12}} & \frac{-\cos\theta_{12}}{R_{12}} & -1  
\end{bmatrix}, \hspace{6mm} \\
H_{21} &=& \begin{bmatrix}
\frac{\sin\theta_{21}}{R_{21}} & \frac{-\cos\theta_{21}}{R_{21}} & -1 & \frac{-\sin\theta_{21}}{R_{21}} & \frac{\cos\theta_{21}}{R_{21}} & 0  
\end{bmatrix}. \hspace{6mm}
\end{eqnarray} 
Define 
\begin{eqnarray}
f_v &=& \begin{bmatrix}
\cos{\psi_1}~ \sin{\psi_1} ~0~ \cos{\psi_2} ~\sin{\psi_2} ~0 
\end{bmatrix}',\nonumber\\
f_{\omega_1} &=& \begin{bmatrix}
0 \quad
0 \quad
1\quad
0 \quad
0\quad
0 
\end{bmatrix}',\nonumber\\
f_{\omega_2} &=& \begin{bmatrix}
0 \quad
0 \quad
0\quad
0 \quad
0\quad
1 
\end{bmatrix}',\nonumber
\end{eqnarray}
and the dynamics can be represented as
\begin{equation}
\dot{X} = \begin{bmatrix}
\dot{X}_1 \\
\dot{X}_2
\end{bmatrix}
= f_vv+f_{\omega_1}\omega_1 + f_{\omega_2}\omega_2,
\end{equation}
and a similar procedure to the range measurement case is followed to find the first order Lie derivatives, observability grammian and covariance matrix $P$. The covariance in position of the first vehicle is given by
\begin{eqnarray}
\sigma_{p_1}^2 = \frac{9}{2}R_g^2\left(1+\frac{R_g^2}{2+R_g^2+2\cos(2(\psi_1-\theta_g))} \right)+\nonumber\\ (R_g^2+R_g^4)\csc^2(\psi_1-\theta_g). \nonumber
\label{eq:s_bearing}
\end{eqnarray} 
Similarly, $ \sigma_p $ of other vehicles can also be found. This approximate closed form covariance is used in the NMPC as Step 2 in Fig. \ref{fig:block}(a).

Next section presents the complete NMPC formulation combining the MHE scheme given in Sec.~\ref{sec:MHE} and the uncertainty results derived from analysis given in Sec.~\ref{sec:analysis}. 
	\section{NMPC formulation}\label{sec:NMPC}
NMPC is a state-of-the-art technique for real-time optimal control. At each time step, the constrained optimization problem is solved based on the plant model for a finite time horizon, and the procedure is repeated with states updated through feedback in the next iteration~\cite{MPCbook}. Fig.~\ref{fig:block} shows the block diagram of the NMPC scheme used in this paper. The optimal control sequence is computed for the prediction horizon $ \tau_h $ from which only the first action is applied to the system at each time step. 
	 
	
 A point mass kinematic model is considered for the vehicles. We assume that the altitude and velocities of the AAVs remain constant during transit. The general kinematic model is given as
	\begin{equation}
	\dot{X} =
	\begin{bmatrix}
	v\cos \psi_1   \\
	v\sin \psi_1\\
	\omega_1 \\
	\vdots \\
	v\cos \psi_{n_v}   \\
	v\sin \psi_{n_v}\\
	\omega_{n_v} \\        
	\end{bmatrix},
	\end{equation}  
 where  $ v $ is the linear velocity, $ \psi $ the heading angle, $ \omega $ the angular velocity, and $ n_v $ is the number of vehicles.
	
The objective function for the NMPC is defined as
	\begin{equation}
	\min_{\boldsymbol{\omega_1 \cdots \omega_{n_v}} \in \mathcal{PC}(t,t+\tau_h)}J=\int_{t}^{t+\tau_h}\sum_{i=1}^{n_v} \left[ C_{1_i}+W_iC_{2_i}\right], 
	\label{eq:PI}
	\end{equation}	
	subject to:
	\begin{align*}
	\dot{X} &= f(X,\boldsymbol{\omega}),\\
	\boldsymbol{\omega} & \in\left[ \boldsymbol{\omega}^{-},\boldsymbol{\omega}^{+}\right], 
	\end{align*}
	where 
	\begin{equation}
	C_{1_i}=(x_i-x_{D_i})^{2}+(y_i-y_{D_i})^{2}, \label{eq:C1}
	\end{equation}
is the cost associated with minimizing the distance between the vehicle and the destination. $(x_i,y_i) $ is the position of the $ i^{th} $ vehicle and $ (x_{D_i},y_{D_i}) $ are their respective destination points. $ \boldsymbol{\omega}^{-} $ and $ \boldsymbol{\omega}^{+}$ are the lower and upper bounds of $\boldsymbol{\omega}$, and $ \mathcal{PC}(t,t+\tau_h) $ denotes the space of piece-wise continuous function defined over the time interval $ \left[ t,t+\tau_h\right]  $. 
$C_{2_i}$ is the cost to ensure the estimation covariance is within a  bound. It is defined as:
	\begin{equation}
	C_{2_i}=\begin{cases}
	0, & \text{if $\lambda_i \geq \eta$}.\\
	(\eta-\lambda_i)^{2}, & \text{ otherwise}.
	\end{cases} \label{eq:C2}
	\end{equation} 
where $ \eta $ is a tuning parameter related to the number of connections required. Increasing $ \eta $ will result in vehicles moving closer to the landmarks and increase the connections. For satisfactory localization, observability conditions should be satisfied, which require connections with at least two landmarks~\cite{sharma2014observability}. Hence, the value of $\eta$ should be selected as~$ \eta \geq 2 $. The parameter $\lambda_i $ is the second smallest eigenvalue of the Laplacian matrix which is formed as:
	\begin{equation}
	L_i(X)=\Delta_i(X)-A_i(X),
	\end{equation} 
where $ A_i(X) $ is the adjacency matrix defined similar to~\cite{de2006decentralized} as
	\begin{equation}
	A_{i_{mn}}=\begin{cases}
	e^{\frac{-\kappa(||R_{mn}||-\rho)}{R_{s}-\rho}}, & \text{$\left| \left| R_{mn}\right| \right| \leq R_s$}.\\
	0, & \text{$\left| \left| R_{mn}\right| \right| > R_s$}.
	\end{cases}
	\end{equation} 
	and
	$ \Delta_i(X) $ is a diagonal matrix with elements
	\begin{equation}
	\Delta_{i_{mm}}=\sum_{n=1}^{N}A_{i_{mn}}
	\end{equation}
where $ m,n= 1  $ to $ N $, and $ N $ is the number of nodes of the graph connecting the vehicles and landmarks. $ \kappa $ is a constant which determines the convergence rate of the exponential function, and $\rho$ is used to set a minimum distance between the landmarks and the vehicles to avoid collisions. $ ||R_{mn}|| $ is the distance between the $ m^{th} $ and $ n^{th} $ nodes, and $ R_s $ is the sensor range of the vehicles. This formulation of the adjacency matrices, rather than updating it with binary values, helps drive the vehicles closer to the landmarks than just maintaining the connections by keeping them in the sensor range and reducing the distance between the landmarks and the vehicles help in decreasing the estimation covariance as explained in sec.~\ref{sec:analysis}. The maximum value $ \lambda_i $ can take is equal to the number of connections of each vehicle, and this insight is used in formulating (\ref{eq:C2}).   
	
The weight $W_i$ associated with $ C_{2_i} $ is defined in the following way
	\begin{equation}
	W_i=\begin{cases}
	W, & \text{if $3\sigma_{p_i} \geq \sigma_c$}.\\
	0, & \text{ otherwise}.
	\end{cases}\label{eq:Wi}
	\end{equation} 
where $\sigma_{p_i}$ is the standard deviation in the estimated position of the $ i^{th} $ vehicle and the constant $\sigma_c$ is the specified critical value. This adaptive weight formulation is used to obtain a trade-off between the two objectives. $\sigma_{p_i} $s are calculated using the expression
\begin{eqnarray}
\sigma_{p_i}^2 = \frac{9}{2}R_g^2\left(1+\frac{R_g^2}{2+R_g^2+2\cos(2(\psi_i-\theta_g))} \right)+\nonumber\\(R_g^2+R_g^4)\csc^2(\psi_i-\theta_g), \nonumber
\label{eq:s_bearing_i}
\end{eqnarray} 
which is explained in detail in section \ref{sec:b_measurement}. The terms $C_1$ and $ C_2 $ of the objective function (\ref{eq:PI}) is normalized as follows:
	\begin{eqnarray}
	C_1(t)&=&\left(\frac{C_1(t) - \min(C_1(t))}{\max(C_1(t)) - \min(C_1(t))}\right),\label{eq:c1t}\\
	C_2(t)&=&\left(\frac{C_2(t) - \min(C_2(t))}{\max(C_2(t)) - \min(C_2(t))}\right).\label{eq:c2t}
	\end{eqnarray}
The NMPC objective function given equation \eqref{eq:PI} uses the expressions given in equations \eqref{eq:C1}\eqref{eq:C2}\eqref{eq:Wi}\eqref{eq:c1t} and equation \eqref{eq:c2t}. The NMPC objective function is solved along with state and control constraints. 
	
	\section{Results and Discussion}\label{sec:results}
Extensive numerical simulations were carried out for validating the proposed scheme using CasADi-Python~\cite{Andersson2019}. We use the total path length  and average estimation error as metric for analyzing the performance of the proposed approach. We perform the following analysis (i) effect of horizon length in the NMPC on the path length and estimation error (ii) effect of cooperation  (iii) comparison with the approach proposed in \cite{manoharan2019nonlinear} and (iv) effect of increasing the number of vehicles in the region to 10. Before presenting the analysis, we will describe the simulation setting.

\subsection{Simulation setup}
We consider an environment of 200m $\times$ 200m, where 20 landmarks are randomly placed. Each vehicle starts at a given location and has a desired goal location. The vehicles have a constant velocity of 5\,m/s.
The angular velocities of the agents are constrained by $[-\pi/2,\pi/2]$\,rad/s due to the practical considerations on the turn rate of the agents. The value of $\eta$ is selected as 2, and the values of $k$ and $\rho$ are selected as 5 and 0.5, respectively. The weight $W$ is selected as 10000, sensor range of the vehicles, $R_s$ = 50\,m, and $\sigma_c$ = 3\,m.  The measurement noise covariance matrix $\Gamma$ is selected as a $n_\Gamma \times n_\Gamma$ matrix with 0.01 in its diagonals, where $n_\Gamma$ is the number of received measurements. Each time step is $0.1$s for all the simulations. The simulations were carried out on a Ubuntu 18.04, Intel i9 workstation with 64GB RAM. 
\subsection{Effect of NMPC horizon length}
In NMPC, the horizon length plays a key role between path optimality and computational time. The larger the horizon, better the path obtained at the cost of increased computational time. This effect can be seen in Fig. \ref{fig:h10}, where the computation time for prediction horizon of $\tau_h=1$\,s is 0.056\,s, however, the paths are not optimal. With increase in $\tau_h$ to 15s, there is significant improvement in the path of vehicle 1 at increased computational time of 0.42s per iteration. With further increase in horizon length to $\tau_h=40$\,s, the average time to compute an iteration is 3.43\,s, but the obtained path length for the vehicles  is near-optimal. The path length for $\tau_h=25s$ is close to that obtained with $\tau_h=40$ but takes only 1.2s.

Further, we conducted Monte-Carlo simulation to see the effect of placement of landmarks on the vehicle paths. Figure \ref{fig:monte} shows the effect of change in landmark placement for different $\tau_h$. From the Fig.\ref{fig:monte}(a), we can see that the computation time for $\tau_h=1$s is very less but the vehicle which reaches the destination at the last on average is also high as shown in \ref{fig:monte}(b). Note that, on average, the times taken by the last agent reaching its goal for $\tau_h=$25s and $\tau_h=$40s are almost similar, however, $\tau_25$ takes far less time. 
Hence we consider $\tau_h=25$s for the  rest of the simulations
\begin{figure*}
	\centering 
	\begin{subfigure}{4.25cm}
		\includegraphics[width=4.25cm]{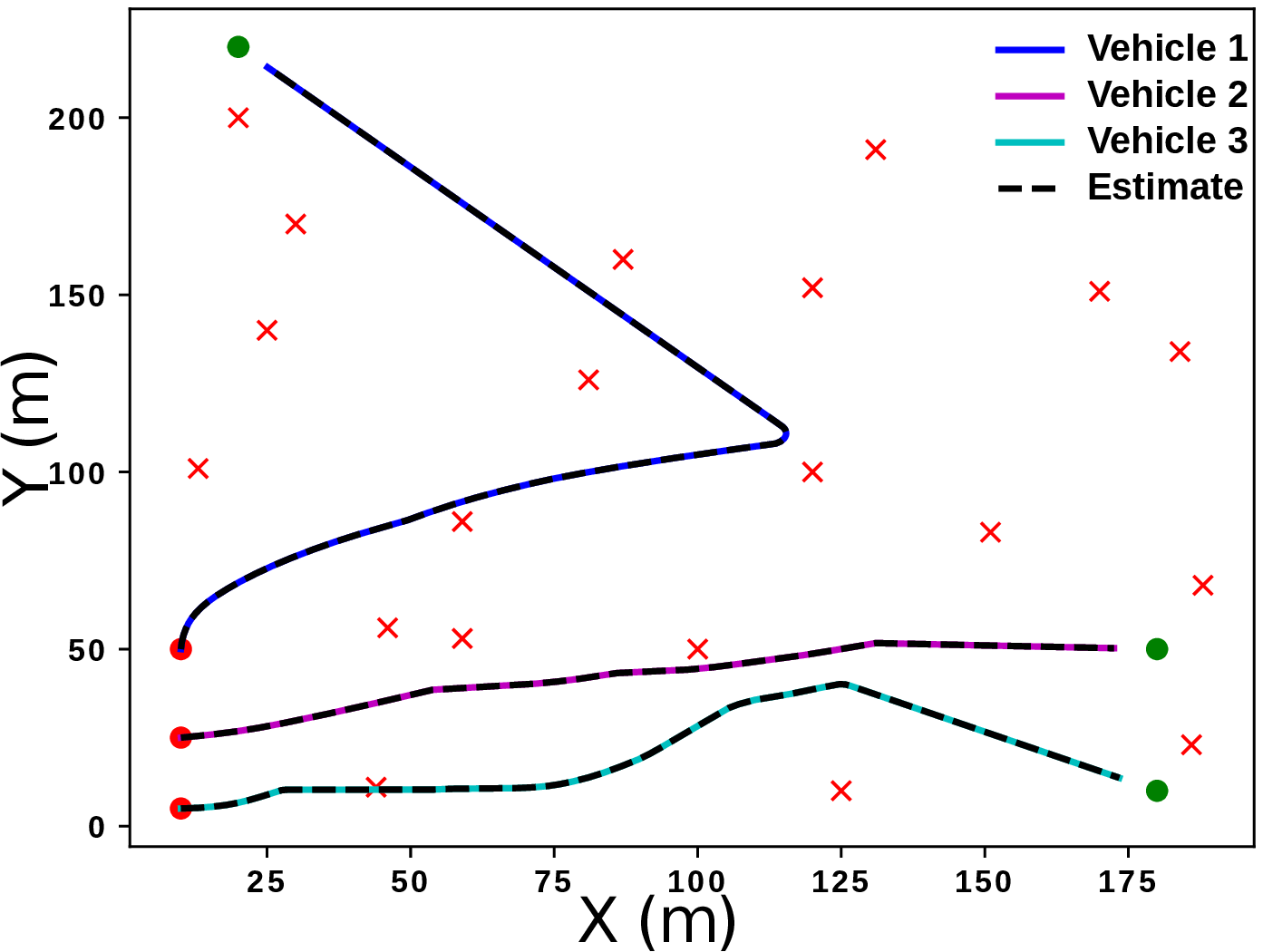}
		\caption{$\tau_h=1$\,s}
		\label{fig:h10}
	\end{subfigure}
	\begin{subfigure}{4.25cm}
		\includegraphics[width=4.5cm]{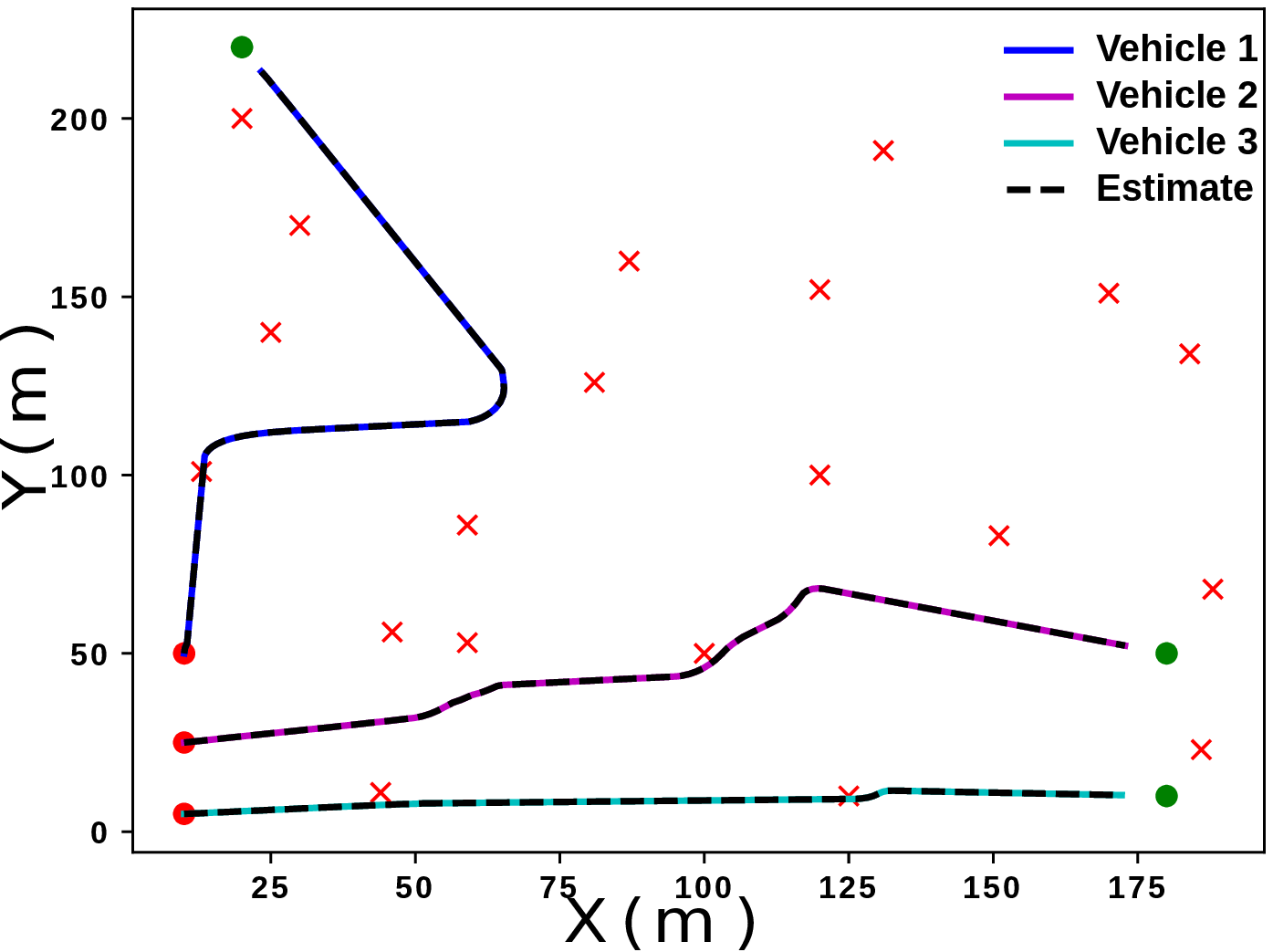}
		\caption{$\tau_h=15$\,s}
		\label{fig:h150}
	\end{subfigure}\hspace{2mm}
	\begin{subfigure}{4.25cm}
		\includegraphics[ width=4.25cm]{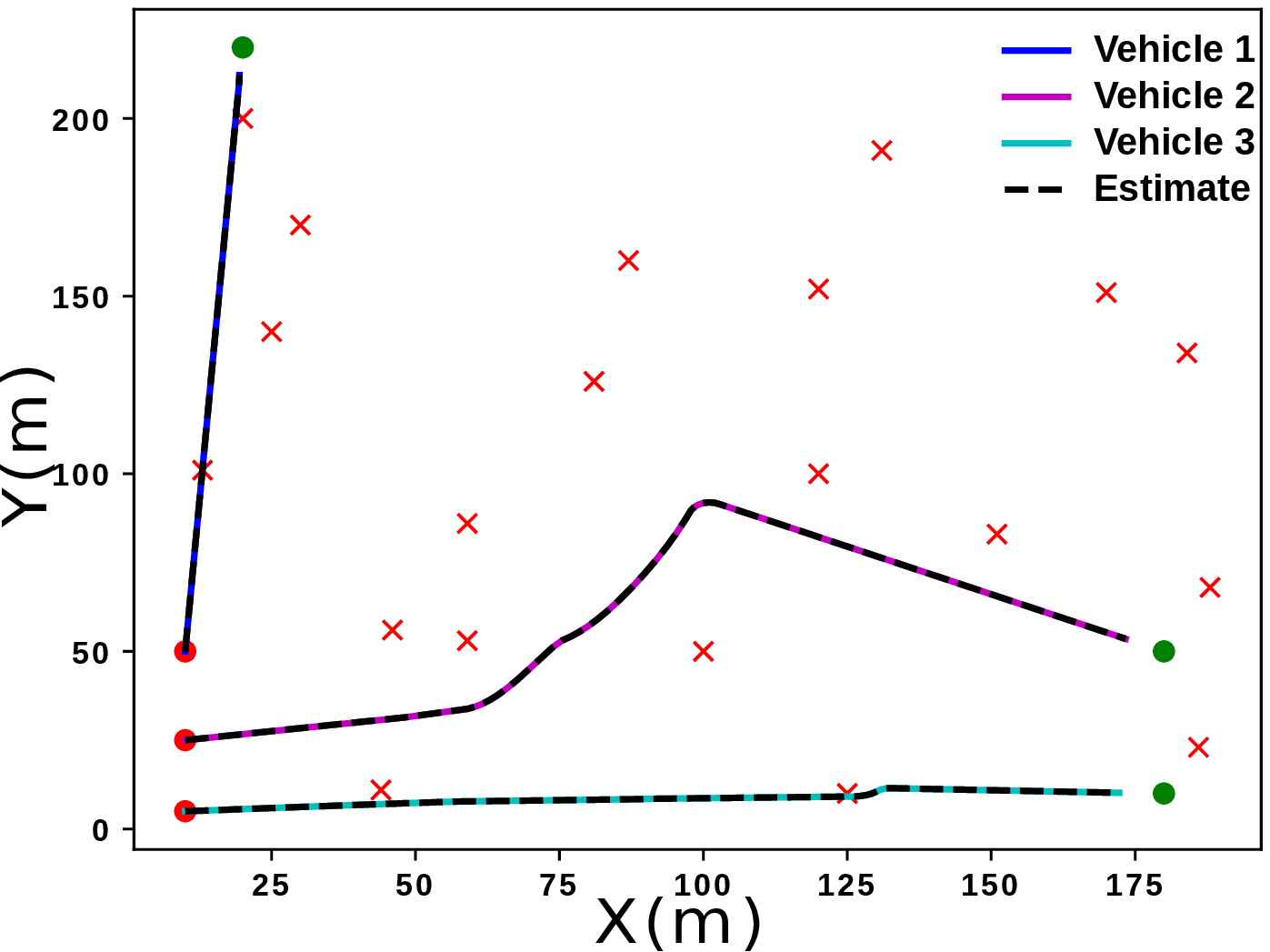}
		\caption{$\tau_h=25$\,s}
		\label{fig:h250}
	\end{subfigure}
	\begin{subfigure}{4.25cm}
		\includegraphics[width=4.25cm]{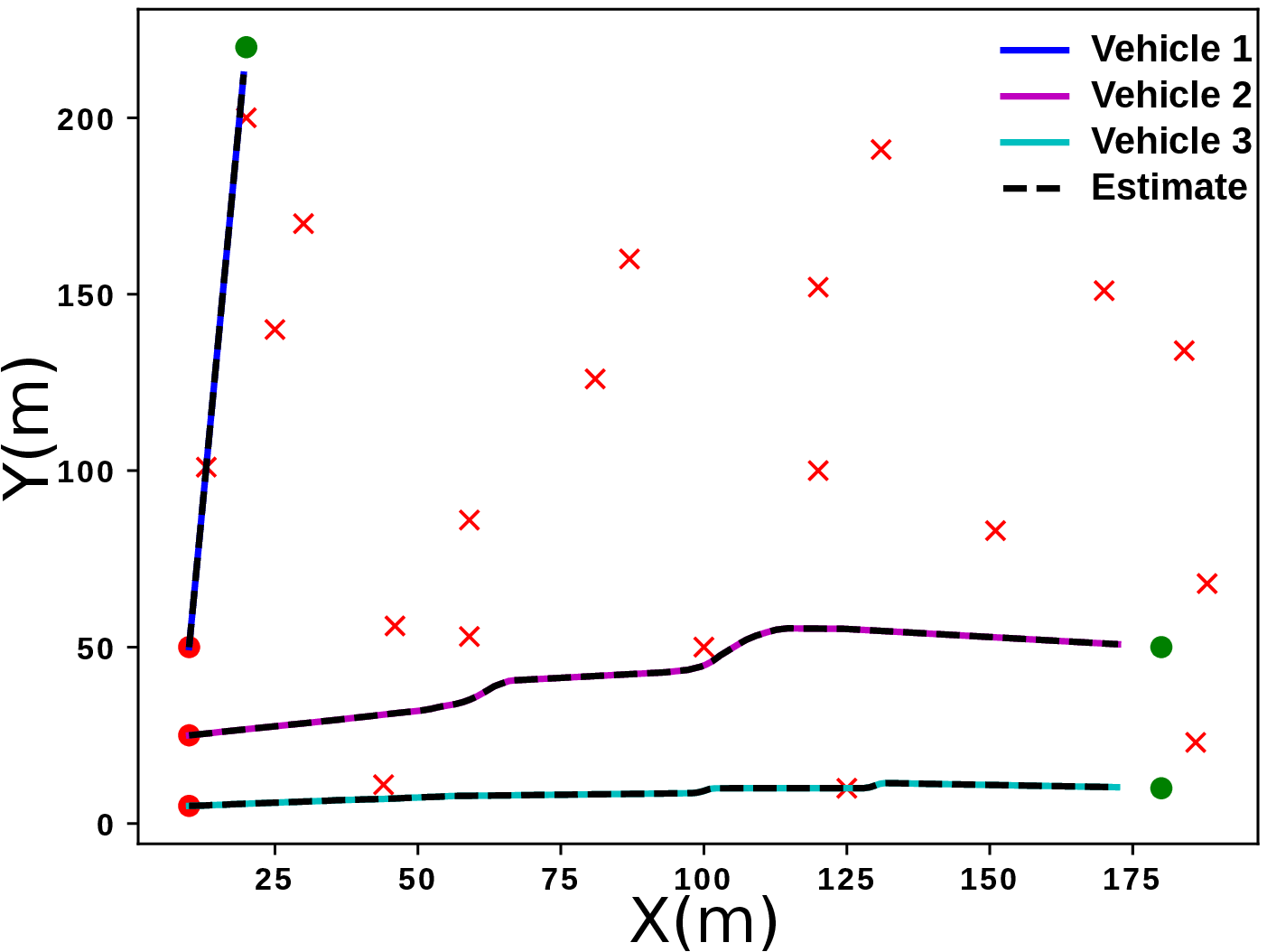}
		\caption{$\tau_h=40$\,s}
		\label{fig:h400}
	\end{subfigure}

	\caption{The average computational time taken per iteration for different $\tau_h$. (a) 0.05\,s (b) 0.42\,s (c) 1.21\,s (d) 3.43\,s.}
	\label{fig:sim_horizon}
\end{figure*}

\begin{figure}
	\centering
	\begin{subfigure}{4.25cm}
	\includegraphics[width=4.4cm,height=2cm]{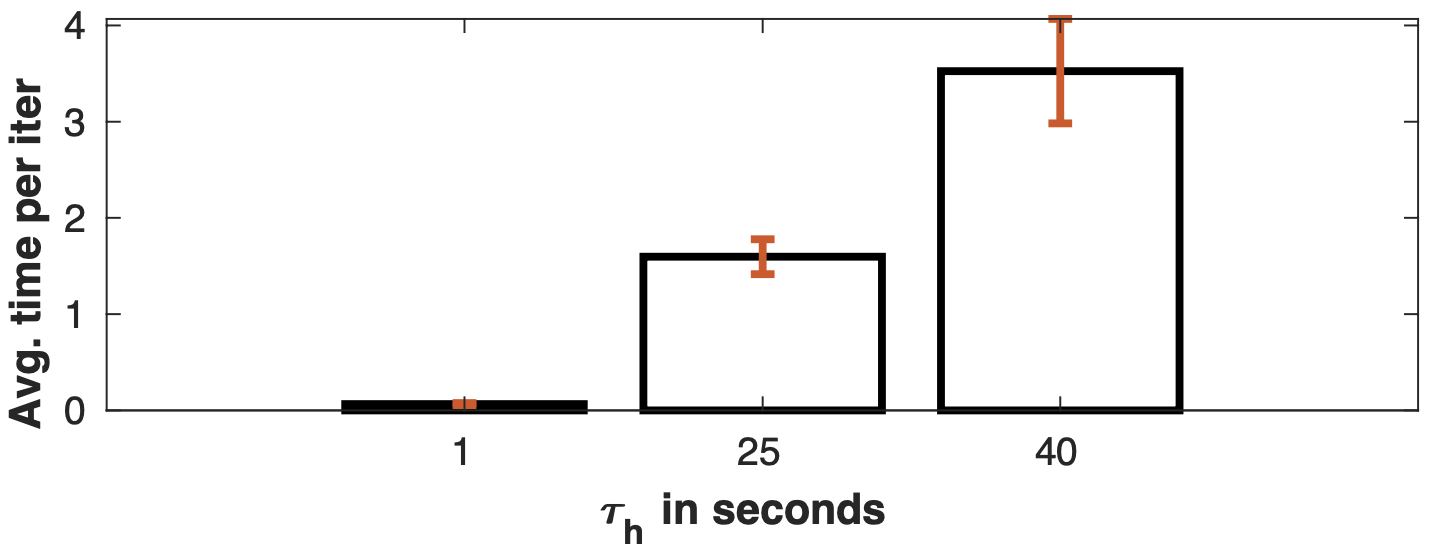}
	\caption{}	
	\end{subfigure}\hspace{2mm}
	\begin{subfigure}{4.25cm}
	\includegraphics[width=4.4cm,height=2cm]{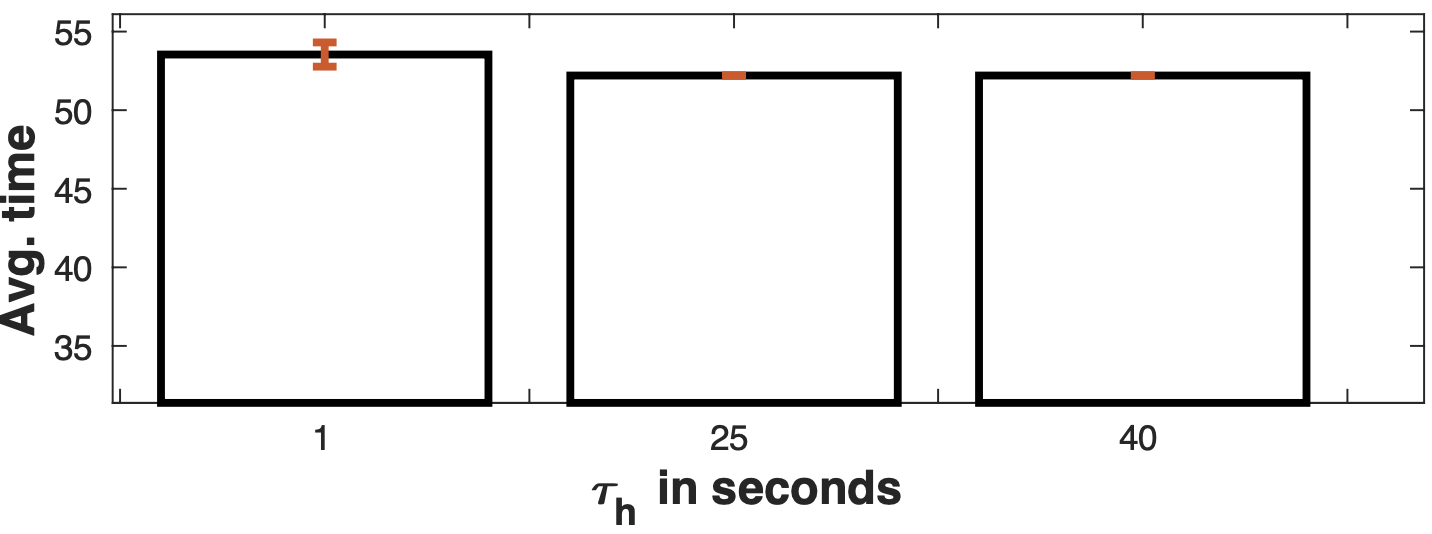}
	\caption{}	
	\end{subfigure}
	\caption{Monte-Carlo simulation for prediction horizon $\tau_h=1,25,40$. (a) Average computational time per iteration (b) Average time taken by all the agents to reach their destinations.} \label{fig:monte}
\end{figure}

\subsection{Effect of cooperation}
One of the main contributions of this paper is to show that with cooperation, the vehicles can jointly determine minimal distance paths to their goal locations while meeting localization accuracy. To show the effect of cooperation, we consider a specific scenario as shown in Fig.~\ref{fig:multi_coop}, where the landmarks are located at the top of the scenario. 
Two simulations were carried out with five vehicles, in which one scenario involved cooperative vehicles and the other without cooperation. We consider  $\tau_h$ = 25\,s, and $R_s=$ 30\,m. The trajectories taken by the vehicles in both the cases are shown in Fig.~\ref{fig:multi_coop}. It can be seen from Fig.~\ref{fig:multi_coop}(a) that the vehicles with cooperation determined a shorter path since they used adjacent vehicles for localization, whereas the non-cooperative vehicles took a longer path to go near the landmarks for localization, as shown in Fig.~\ref{fig:multi_coop}(b). The estimation errors are a bit high for the cooperative case due to the mutual localization of the vehicles.  

\begin{figure}
	\centering 
	\includegraphics[width=4.3cm]{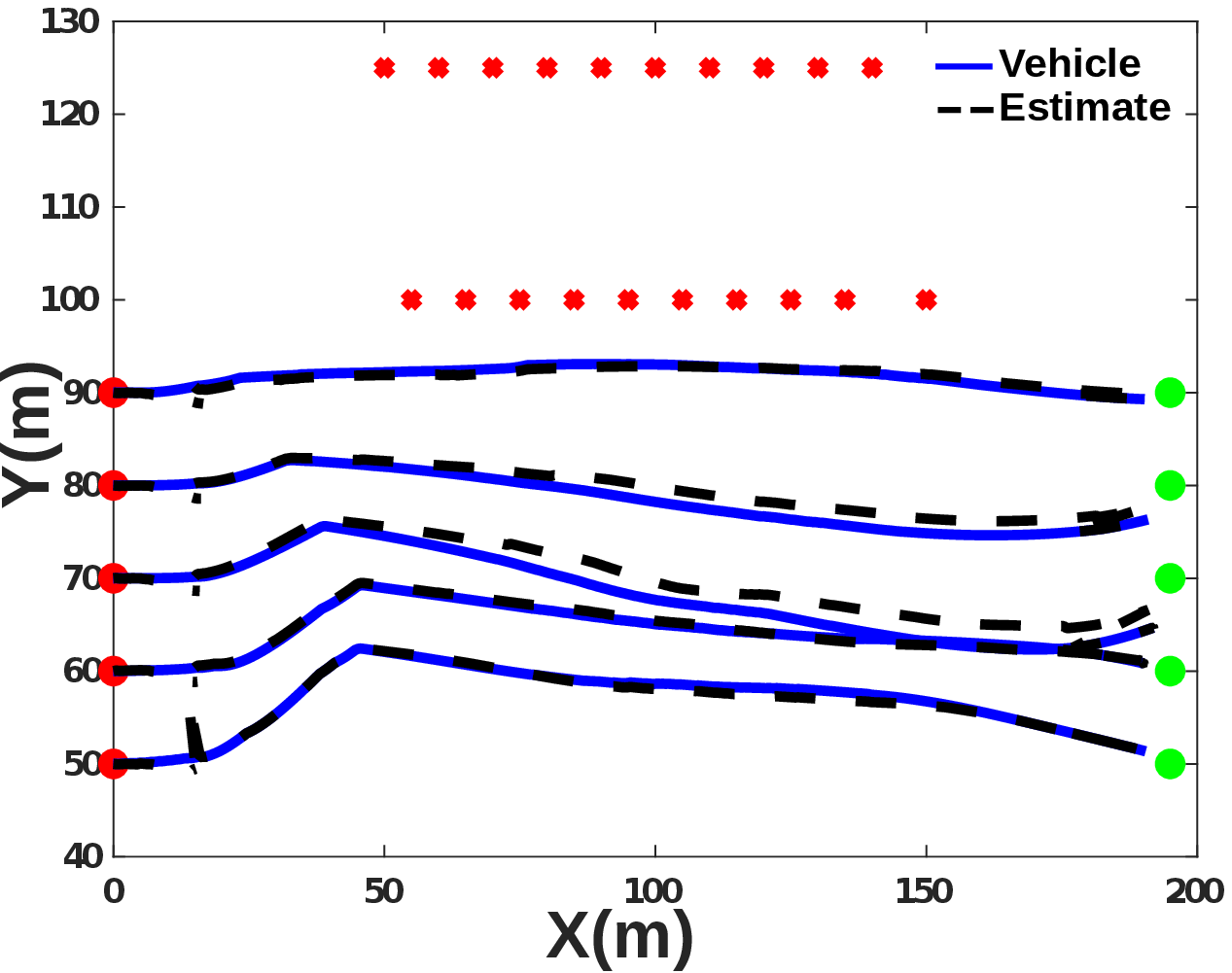}
	\includegraphics[width=4.3cm]{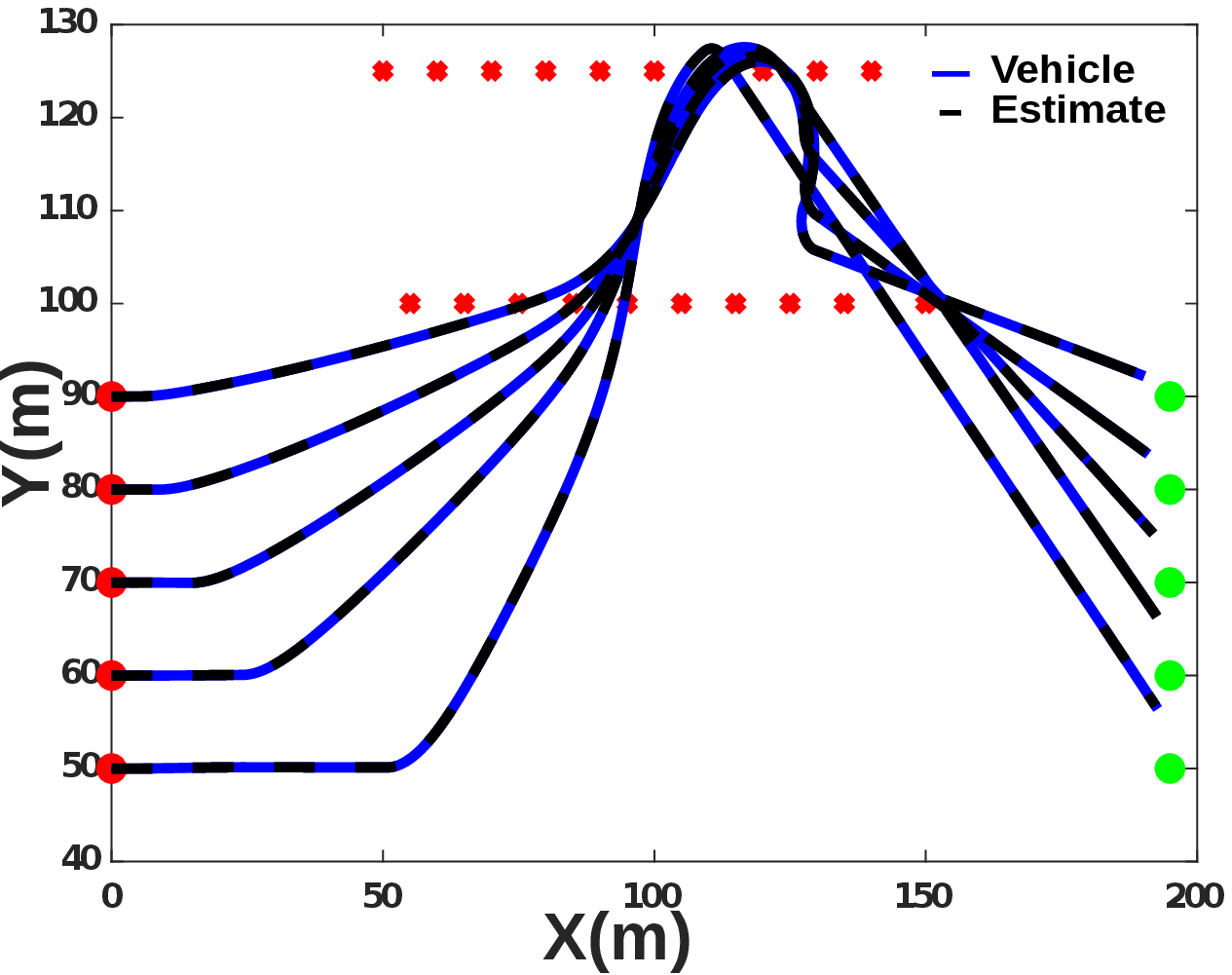}
	\caption{Effect of cooperation. (a) Trajectory of vehicles with cooperation. (b) Trajectory of vehicles without cooperation.}
	\label{fig:multi_coop}
\end{figure}

 \subsection{MHE vs EKF}
 The performance of the MHE estimator was compared against a standard EKF formulation to validate the superiority of the proposed scheme. In Fig. \ref{fig:block}, the step 1 sub-block of the MHE is replaced by the EKF. The horizon length for the MHE is selected as $ N_e=20 $, and all other initial parameters are taken the same for the MHE and the EKF. The estimated trajectories of the vehicles, true trajectories, and the errors in position for both the MHE and the EKF are given in Fig.~\ref{fig:EKF_MHE}. It can be seen that the estimation errors are less for the MHE, and it is also more stable compared to the EKF. The vehicles' mean square error (MSE) for the MHE is 0.46\,m and that with the EKF is 0.79\,m. The computation time requirement for the MHE is 1.21\,s, while the EKF takes 1.10\,s. Although the MHE takes 9\% more computational time than the EKF, its accuracy is 72\% better than the EKF. 

\begin{figure*}
	\centering 
	\begin{subfigure}{3.5cm}
		\includegraphics[width=3.5cm]{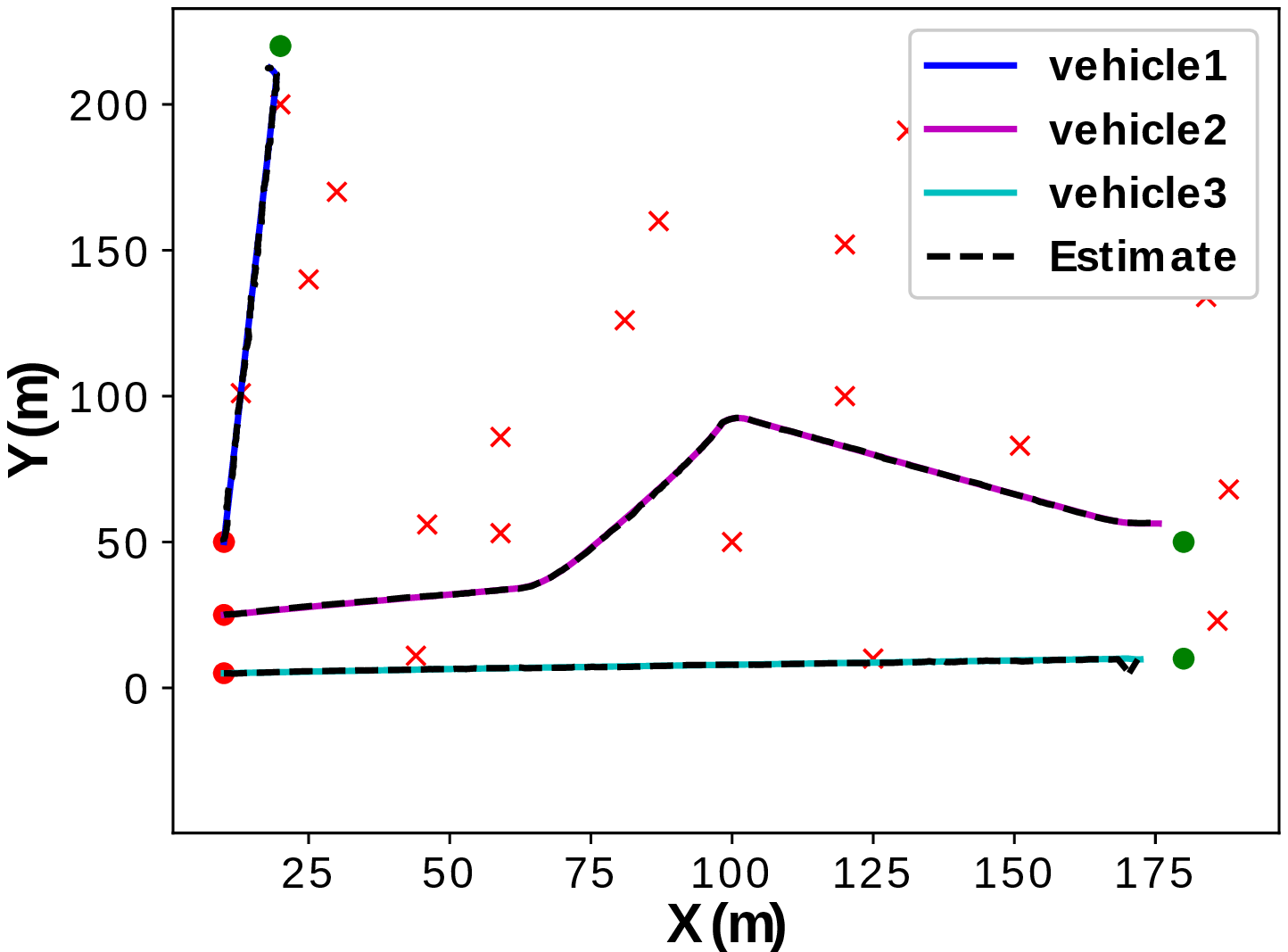}
		\caption{}
		\label{fig:ekf_pc}
	\end{subfigure}
	\begin{subfigure}{3.5cm}
		\includegraphics[width=3.5cm]{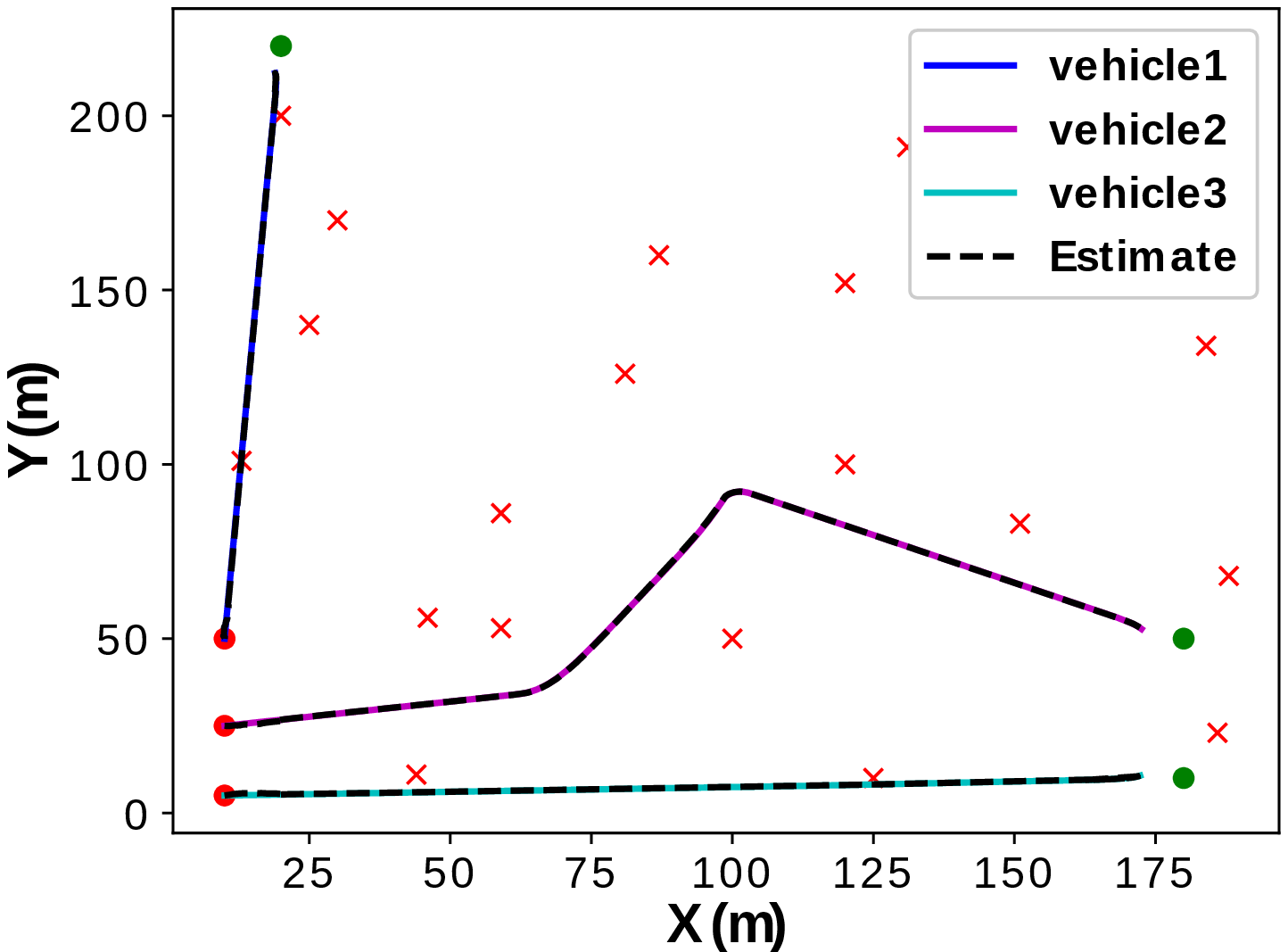}
		\caption{}
		\label{fig:mhe_pc}
	\end{subfigure}
	\begin{subfigure}{3.5cm}
		\includegraphics[width=3.5cm]{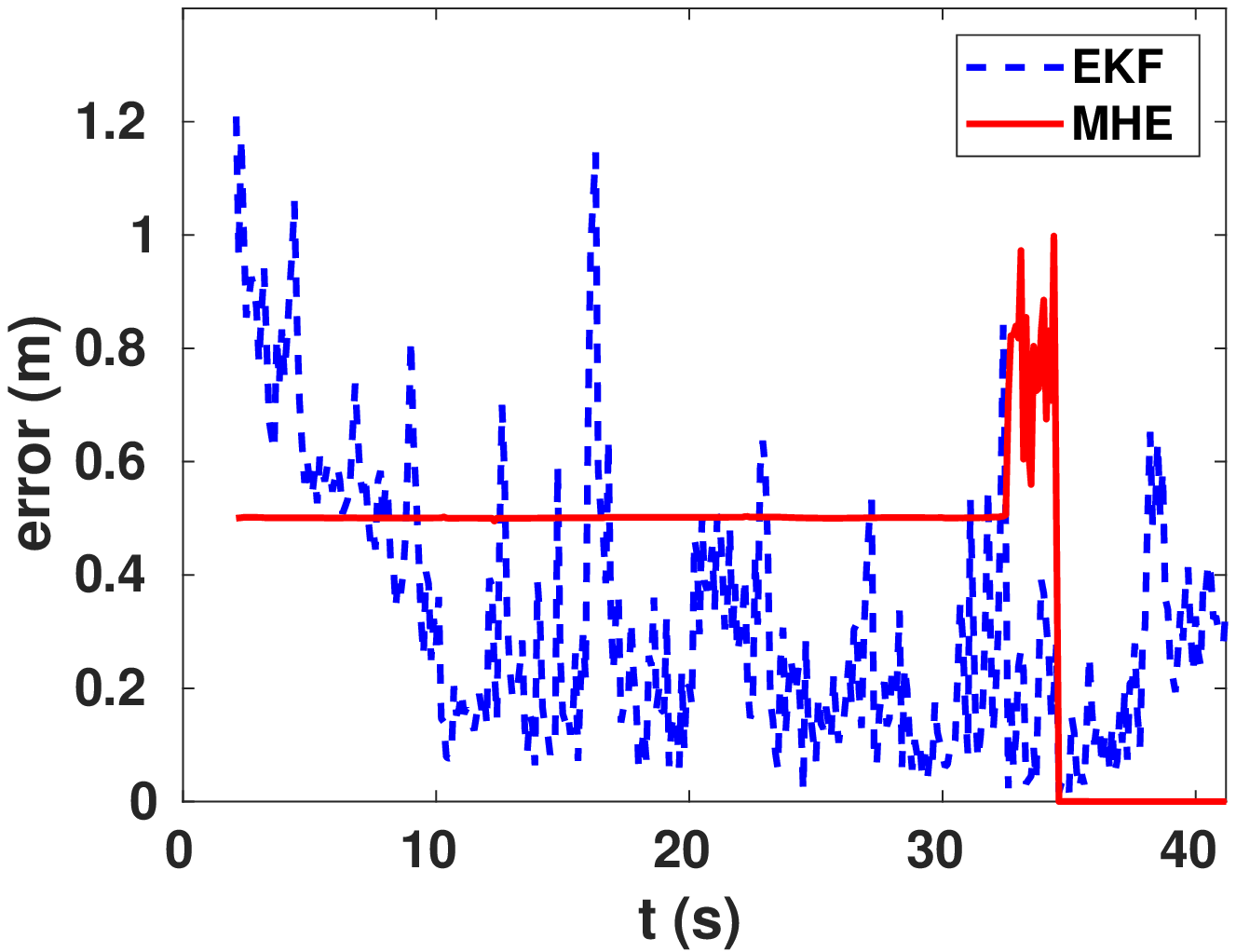}
		\caption{}
		\label{fig:ekfp1}
	\end{subfigure}
	\begin{subfigure}{3.5cm}
		\includegraphics[width=3.5cm]{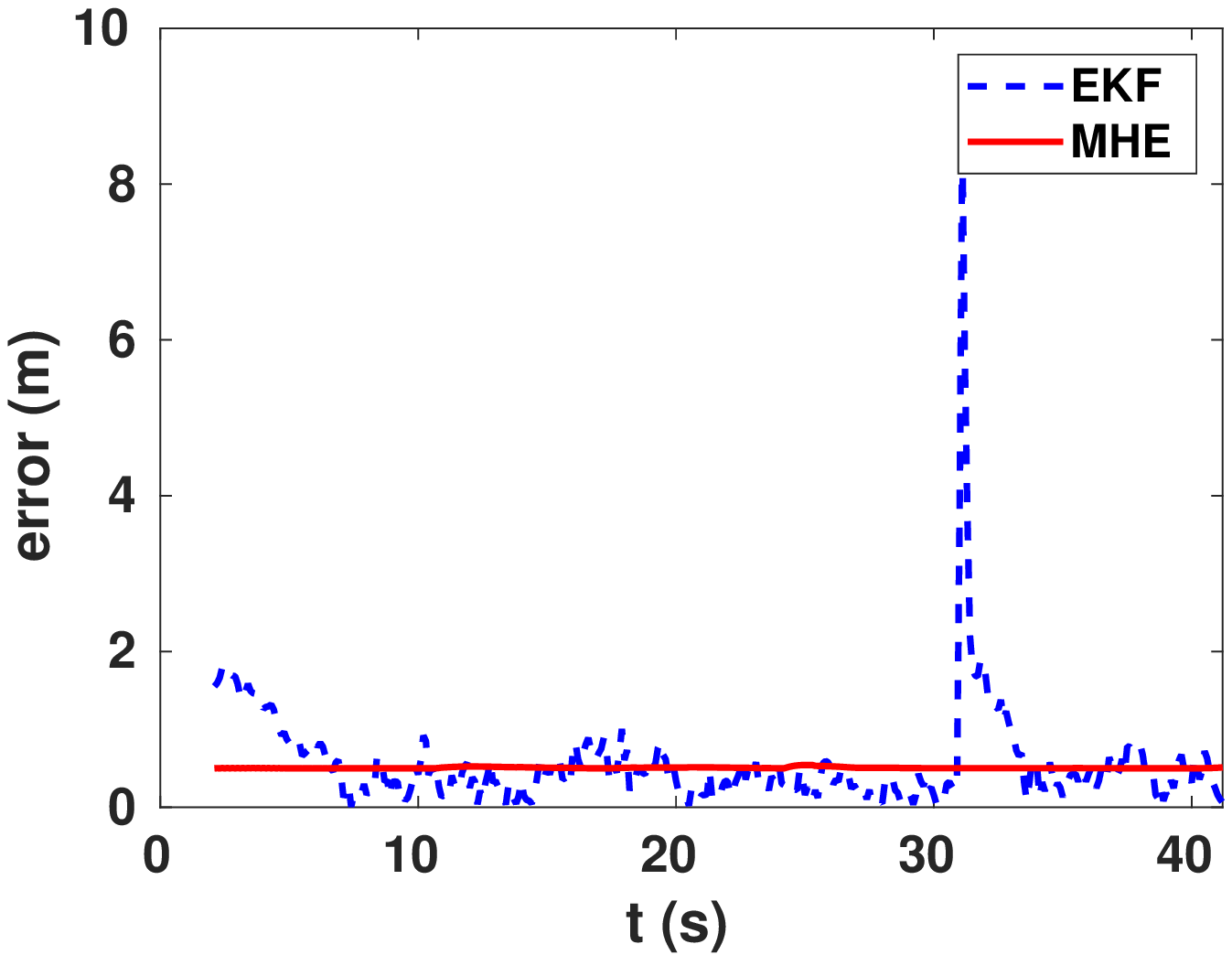}
		\caption{}
		\label{fig:ekfp2}
	\end{subfigure}
	\begin{subfigure}{3.5cm}
		\includegraphics[width=3.5cm]{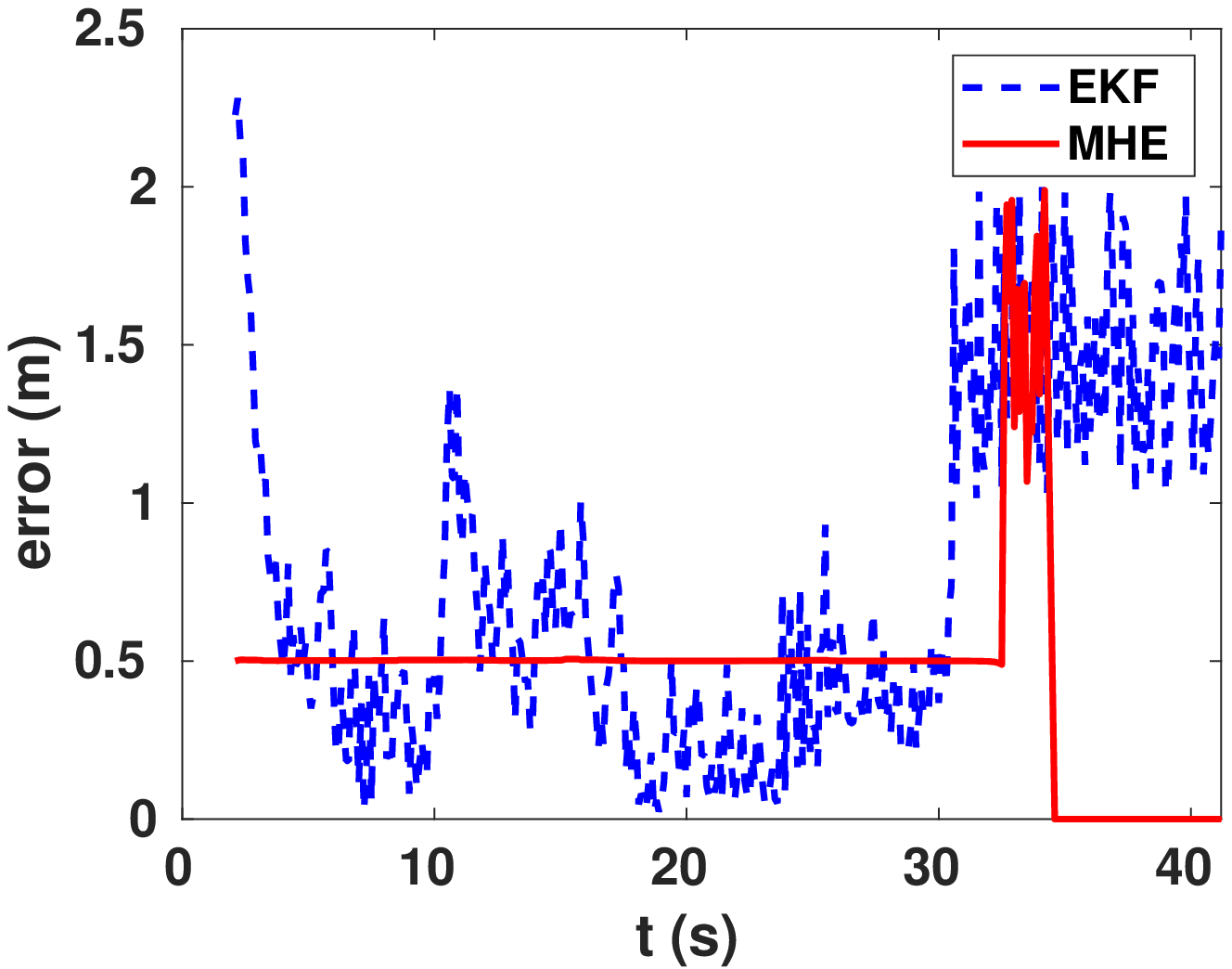}
		\caption{}
		\label{fig:ekfp3}
	\end{subfigure}
	\caption{(a) NMPC solution with EKF based estimation (b) NMPC solution with MHE (c) Error in position of vehicle 1 (d) Error in position of vehicle 2  (e)  Error in position of vehicle 3}
	\label{fig:EKF_MHE}
\end{figure*}


\subsection{10 vehicles}
To test the performance of the scheme for larger systems, a scenario involving ten vehicles in a 500m $\times$ 500m plane with all the vehicles moving at a velocity of 10\,m/s was considered. Multiple simulations were carried out with random configurations of the landmarks, initial positions, and goal points. Two example results are shown in Fig.~\ref{fig:multi10}. It can be seen that the estimator performance was satisfactory since the actual states and the estimated states are well aligned for all the ten vehicles. The average computation time per iteration was 7.25\,s due to the increase in the size of state matrix for 10 vehicles.

\begin{figure}
	\centering 
			\includegraphics[width=4.3cm]{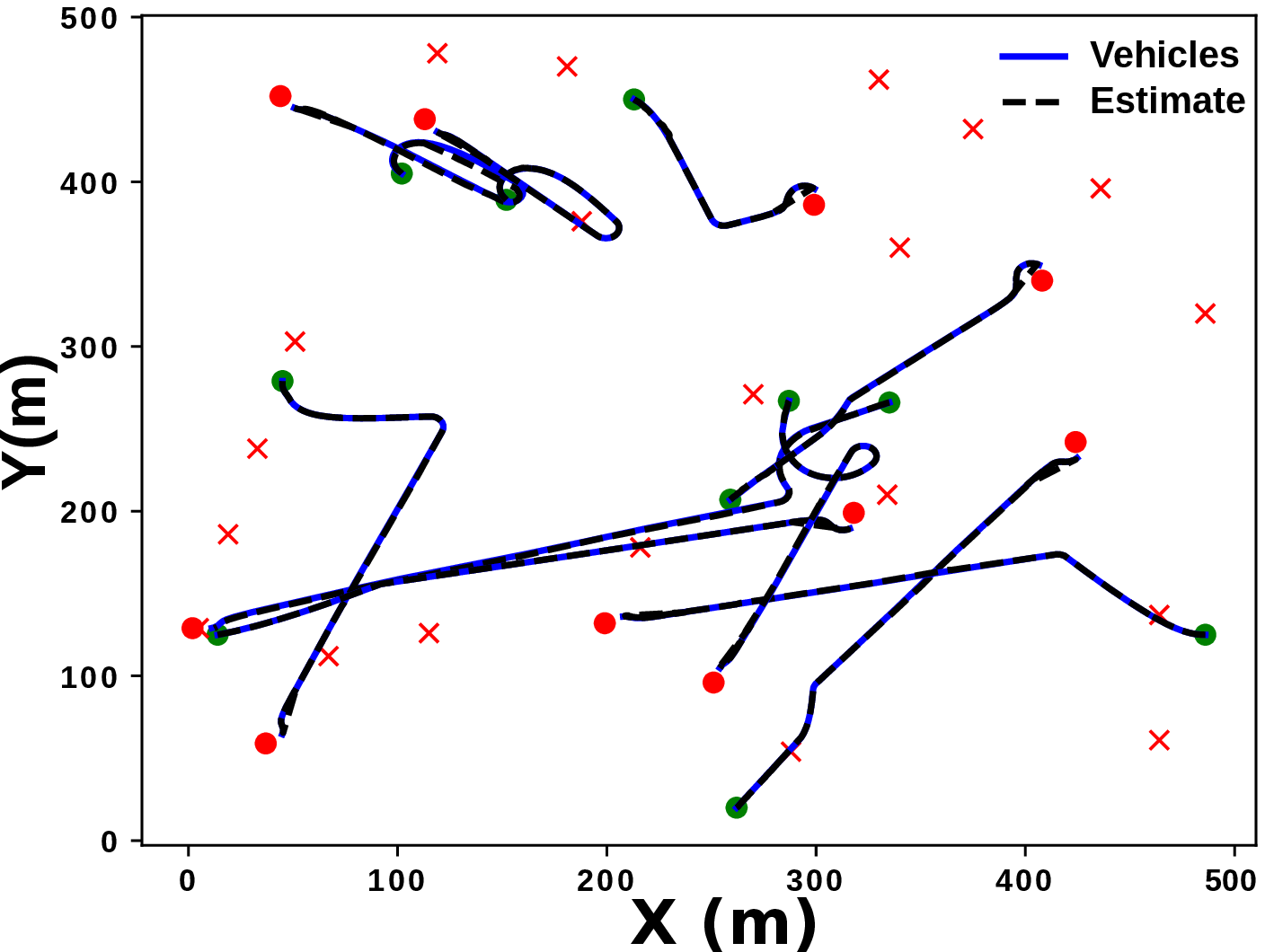}
					\includegraphics[width=4.3cm]{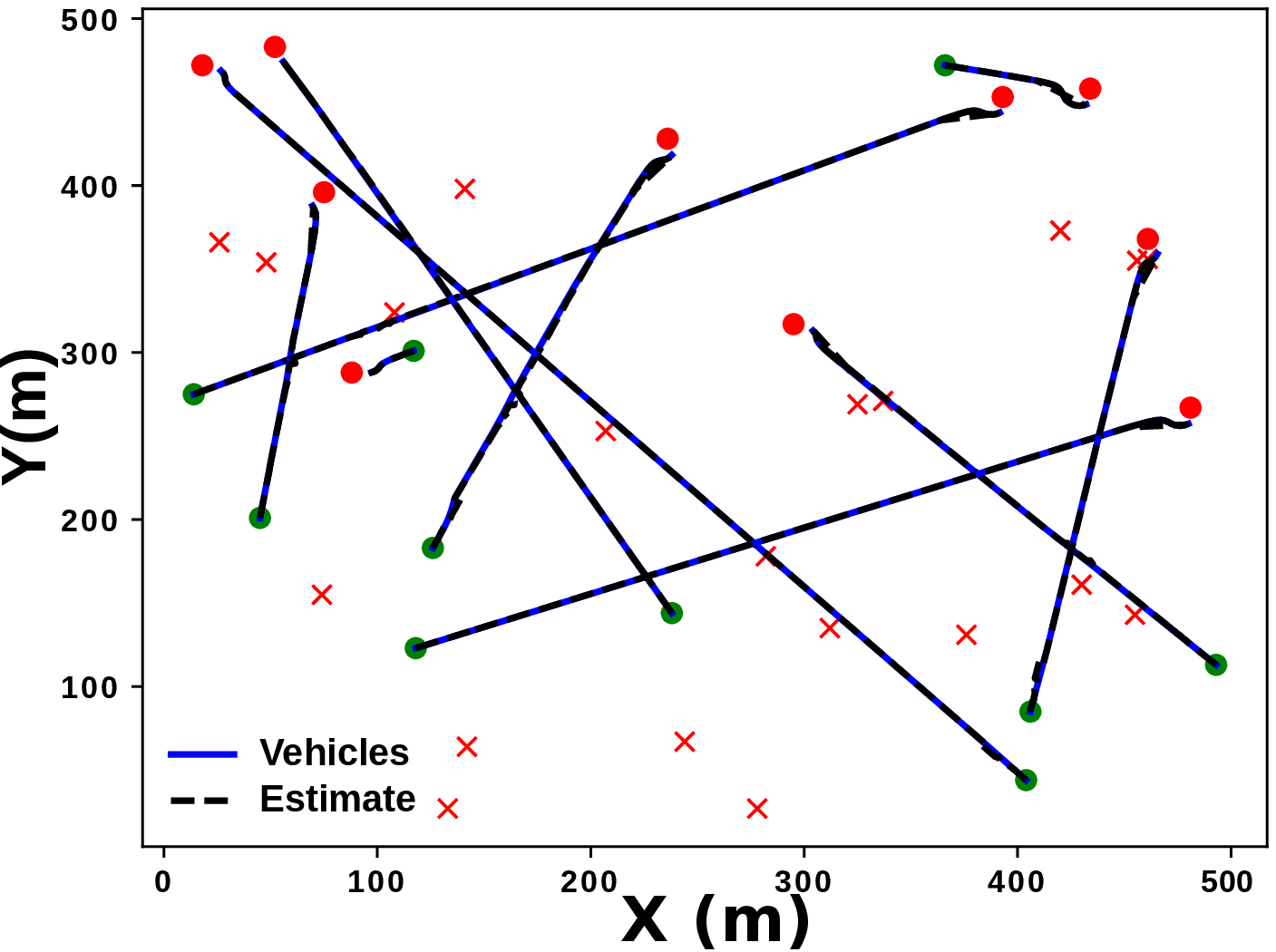}
		\caption{Trajectories of 10 vehicles using NMPC-MHE scheme.}
		\label{fig:multi10}
\end{figure}    
	\section{Conclusions}\label{sec:conclusions}
A nonlinear model predictive control scheme combined with moving horizon estimation was proposed to aid cooperative localization of a group of AAVs in transit. The controller used an approximate analytical expression for calculating the expected covariance of the vehicles through the prediction horizon, which was derived using the insights obtained from analyzing the observability and the path information from the landmark-vehicle graph. The controller determined near optimal paths for the vehicles while satisfying various state and localization constraints.
We analyzed through simulations the role of prediction horizon on the optimality of the vehicle paths and the required computation time.  The proposed moving horizon estimator also outperformed the EKF with lower estimation error values at a small additional computation time. A comparison was performed between cooperative and non-cooperative vehicles to show the significance of cooperation in determining paths under localization constraints. 

The approached proposed in this paper can be extended in several direction. One potential analysis is to determine how many landmarks are sufficient for a given vehicle to reach the destination meeting localization constraints. Another extension can be to include obstacle avoidance which reaching their destination as part of the problem. Experimental validation is another direction to implement the algorithms in real-world vehicle.

	\bibliographystyle{IEEEtran}
	\bibliography{ref}
\end{document}